\documentclass{article}
\usepackage[utf8]{inputenc}
\usepackage{graphicx} 
\usepackage[colorinlistoftodos]{todonotes}
\usepackage{tikz}
\usepackage{pgfplots}
\usepackage{xcolor}
\usepackage{pgfplots}
\usepackage{subcaption} 
\usepackage{amsthm}
\usepackage{amssymb}
\usepackage{amsmath} 
\usepackage{enumitem}
\usepackage{standalone}
\usepackage{quoting} 
\usepackage{hyperref}
\usepackage{cleveref}
\usepackage{tcolorbox}
\usepackage{subcaption} 
\usepackage{tcolorbox}
\usepackage{fancyvrb} 
\tcbuselibrary{breakable}
\usepackage{listings} 
\usepackage{filecontents}
\usepackage{verbatim}
\usepackage{textcase} 
\usepackage{upquote}
\usepackage{tabularx}
\usepackage{colortbl}
\usepackage{array}
\usepackage{graphicx}
\usepackage{subcaption}
\usepackage{capt-of}
\usepackage[english]{babel}
\usetikzlibrary{matrix, backgrounds, positioning, calc, graphs, fit, shapes.geometric, shadows, arrows.meta}
\usepackage{float}
\usepackage{fontawesome5}    
\usepackage{multicol}
\usepackage{graphicx}
\usepackage{colortbl}
\usepackage{booktabs} 
\usepackage{lmodern} 
\usepackage{caption}
\usepackage{tabularray}
\usepackage{longtable}
\usepackage{booktabs}

\usepackage[margin=1.0in]{geometry}

\title{\vspace{-1.0cm}FormulaOne: Measuring the Depth of Algorithmic Reasoning\ \  Beyond Competitive Programming}

\author{
Gal Beniamini \and
Yuval Dor \and
Alon Vinnikov \and
Shir Granot Peled \and
Or Weinstein \and
Or Sharir \and
Noam Wies \and
Tomer Nussbaum \and
Ido Ben Shaul \and
Tomer Zekharya \and
Yoav Levine \and
Shai Shalev-Shwartz \and
Amnon Shashua
}
\date{\textbf{AAI}}

\begin{document}

\maketitle

\newtheorem{theorem}{Theorem}[section]
\newtheorem*{theorem*}{Theorem}
\newtheorem{corollary}{Corollary}[theorem]
\newtheorem{lemma}[theorem]{Lemma}
\newtheorem{proposition}[theorem]{Proposition}
\newtheorem{claim}[theorem]{Claim}
\newtheorem{definition}[theorem]{Definition}
\newtheorem{notation}[theorem]{Notation}
\newtheorem*{openq*}{Question}
\newtheorem{exmp}[theorem]{Example}

\newcommand{\Sn}{\mathbb{S}_n}
\newcommand{\NN}{\mathbb{N}}
\newcommand{\CC}{\mathbb{C}}
\newcommand{\RR}{\mathbb{R}}
\newcommand{\QQ}{\mathbb{Q}}
\newcommand{\ZZ}{\mathbb{Z}}

\newcommand{\eqdef}{\mathrel{\mathop:}=}

\definecolor{lblue}{HTML}{BFD7EA}
\definecolor{otherblue}{HTML}{4e94ba}
\definecolor{terracota}{HTML}{e35336}
\definecolor{beige}{HTML}{DCC5B2}
\definecolor{nicegreen}{HTML}{DDEB9D}
\definecolor{nicered}{HTML}{D84040}
\definecolor{nicepurple}{HTML}{BA487F}
\definecolor{darkgreen}{HTML}{1F7D53}
\definecolor{brightmaroon}{rgb}{0.76, 0.13, 0.28}
\definecolor{tan}{HTML}{EAC8A6}
\definecolor{forbiddenblue}{HTML}{2980B9}
\definecolor{attritiongreen}{HTML}{1ABC9C}
\definecolor{extremalpurple}{HTML}{8E44AD}
\definecolor{ideblue}{HTML}{3F7898}
\definecolor{ideorange}{HTML}{D58747}
\definecolor{idegreen}{HTML}{609B55}
\definecolor{lightred}{HTML}{FFCCCB}
\definecolor{lightgreen}{HTML}{CCFEFF}
\definecolor{dblue}{HTML}{42A5F5}    
\definecolor{lorange}{HTML}{FFF3E0}  
\definecolor{dorange}{HTML}{FFB74D}  
\definecolor{arrowgray}{HTML}{546E7A}
\definecolor{dgreen}{HTML}{66BB6A}   
\definecolor{lgreen}{HTML}{E8F5E9}   
\definecolor{lblue}{RGB}{204, 229, 255}
\definecolor{plotblue}{HTML}{51A6DF}
\definecolor{plotred}{HTML}{E96658}
\definecolor{plotgreen}{HTML}{4CD385}
\definecolor{plotyellow}{HTML}{F4AA36}
\colorlet{colorLight}{lblue!78}
\colorlet{colorBase}{lblue}

\newcommand{\todogal}[1]{\todo[color=lblue]{[Gal]: #1}}
\newcommand{\itodogal}[1]{\todo[color=lblue,inline]{[Gal]: #1}}
\newcommand{\todoyuval}[1]{\todo[color=purple]{[Yuval]: #1}}
\newcommand{\itodoyuval}[1]{\todo[color=purple,inline]{[Yuval]: #1}}
\newcommand{\todoalon}[1]{\todo[color=green]{[Alon]: #1}}
\newcommand{\itodoalon}[1]{\todo[color=green,inline]{[Alon]: #1}}
\newcommand{\todoyoav}[1]{\todo[color=beige]{[Yoav]: #1}}
\newcommand{\itodoyoav}[1]{\todo[color=beige,inline]{[Yoav]: #1}}
\newcommand{\todonadav}[1]{\todo[color=brown]{[Nadav]: #1}}
\newcommand{\itodonadav}[1]{\todo[color=brown,inline]{[Nadav]: #1}}

\newcommand{\problemblock}[3]{
    \begin{tikzpicture}
        \node[draw, thick, rounded corners=3pt, minimum width=2.5cm, minimum height=2.5cm, label=center:#2, color=gray] (graphA) {};
        \node[draw, thick, rounded corners=3pt, minimum width=2.5cm, minimum height=2.5cm, label=center:#3, right=0.3cm of graphA, color=gray] (graphB) {};
        \path (graphA.north) -- (graphB.north) coordinate[pos=0.5] (midpoint);
        \node[font=\bfseries, above=0.2cm of midpoint] (title) {#1};
        \node[below=0.15cm of graphA, text=green!50!black] {\normalsize\checkmark};
        \node[below=0.15cm of graphB, text=red!70!black] {\normalsize\sffamily X};
    \end{tikzpicture}
}


\tikzset{
    thoughtbox/.style={
        rectangle,
        rounded corners=8pt,
        draw=dblue,
        fill=lblue,
        thick,
        drop shadow={opacity=0.2, shadow xshift=2pt, shadow yshift=-2pt},
        text width=6.5cm, 
        align=center,
        inner sep=6pt,
        font=\sffamily\small\linespread{0.9}\selectfont,
    },
    observationbox/.style={
        thoughtbox, 
        fill=lorange,
        draw=dorange,
    },
    iobox/.style={
        thoughtbox, 
        fill=lgreen,
        draw=dgreen,
    },
    arrowstyle/.style={
        -{Latex[length=2.5mm, width=2mm]}, 
        thick,
        draw=arrowgray
    }
}

\begin{abstract}

Frontier AI models demonstrate formidable breadth of knowledge. But how close are they to true human --- or superhuman --- expertise? Genuine experts can tackle the hardest problems and push the boundaries of scientific understanding.
To illuminate the limits of frontier model capabilities, we turn away from contrived competitive programming puzzles,
and instead focus on real-life research problems. 

We construct FormulaOne, a benchmark that lies at the intersection of graph theory, logic, and algorithms, all well
within the training distribution of frontier models.
Our problems are incredibly demanding, requiring an array of reasoning steps, involving topological and geometric insight,
mathematical knowledge, combinatorial considerations, precise implementation, and more.
The dataset has three key properties. First, it is of commercial interest and relates to practical large-scale optimisation
problems, such as those arising in routing, scheduling, and network design. Second, it is generated from the highly expressive framework of Monadic Second-Order (MSO) logic on graphs, paving the way toward automatic problem generation at scale ---  ideal for building RL environments. Third, many of our problems are intimately related to the frontier of 
theoretical computer science, and to central conjectures therein, such as the Strong Exponential Time Hypothesis (SETH). As such, 
any significant algorithmic progress on our dataset, beyond known results, could carry profound theoretical implications.

Remarkably, state-of-the-art models like OpenAI’s o3 fail entirely on FormulaOne, solving less than 1\% of the questions, even when given 10 attempts and explanatory fewshot examples
--- highlighting how far they remain from expert-level understanding in some domains. To support further research, we additionally curate FormulaOne-Warmup, offering a set of simpler tasks, from the same distribution. We release the full corpus along with a comprehensive evaluation framework.
\end{abstract}
\begin{figure}[H]  
\centering
\begin{center}
\includegraphics[width=0.74\linewidth]{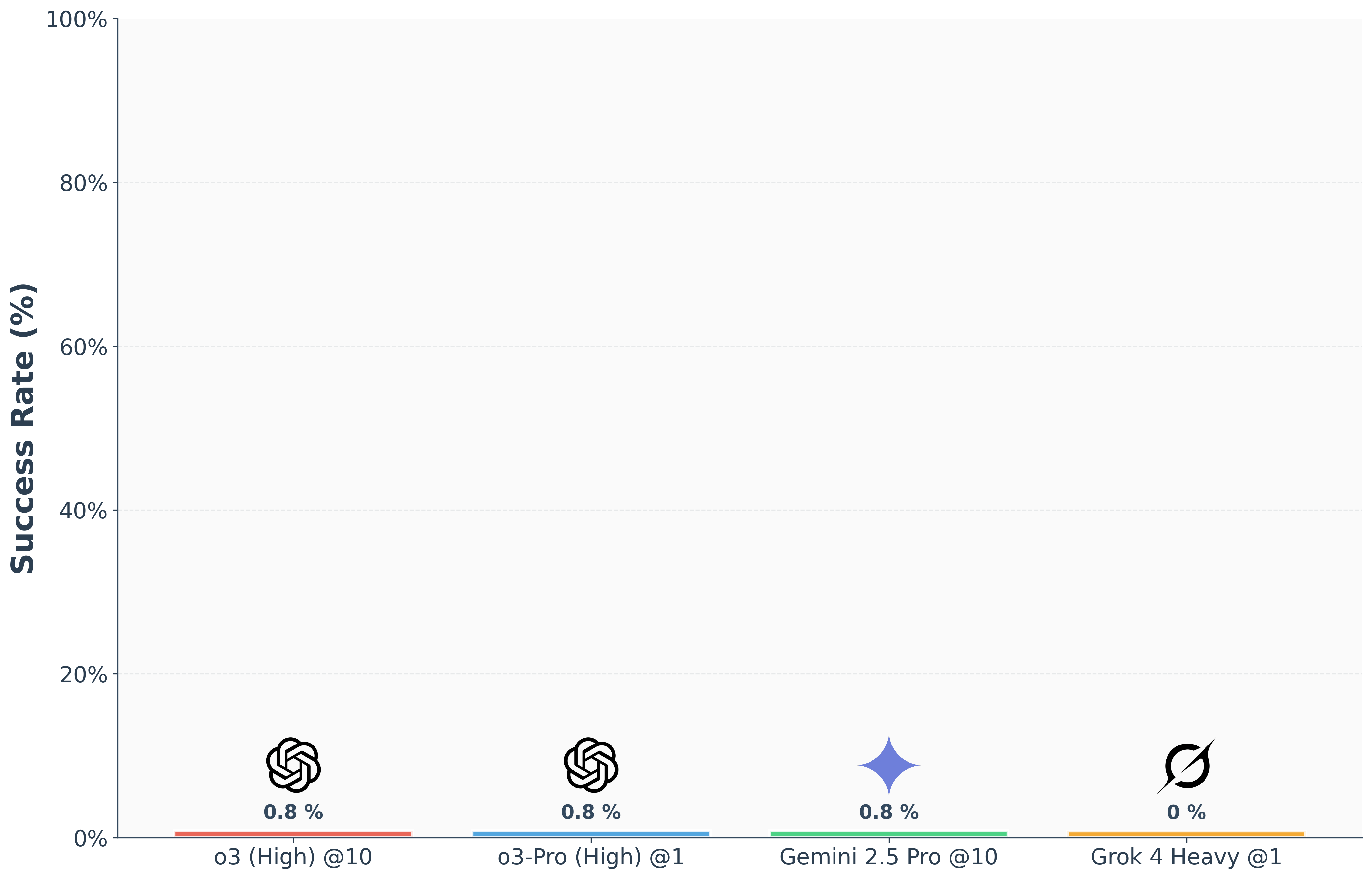}
\caption{Performance of frontier reasoning models on the FormulaOne dataset.}
\label{fig:Abstract}
\end{center}
\end{figure}

\section{Introduction}
\label{sec:introduction}

Artificial Intelligence (AI) holds the promise of solving the world's hardest scientific, algorithmic, and mathematical challenges---problems so complex they baffle even the brightest human minds. Current benchmarks, however, often do not paint a complete picture of AI's depth of understanding. While recent achievements are remarkable, such as OpenAI-o3 attaining a 2,724 rating on CodeForces or securing a gold medal at the International Olympiad in Informatics~\cite{el2025competitive}, they nevertheless mask a sobering reality: the skills honed for these competitions do not capture the full spectrum of reasoning needed for large-scale, real-world research problems. Tasks such as optimising global supply chains, managing large-scale power grids, and designing resilient network infrastructures are orders of magnitude harder, requiring algorithmic insight that goes far beyond the scope of typical competitive programming.

To this end, we introduce \textbf{FormulaOne}, a benchmark centred around dynamic programming over graphs---an algorithmic cornerstone of real-world optimisation. Our framework is constructed in a principled, semi-mechanistic manner based on Monadic Second-Order (MSO) logic, a formal logic on graphs. Its theoretical foundation is an algorithmic meta-theorem due to Courcelle \cite{courcelle1990monadic}, which guarantees that a vast class of problems defined using this formal logic can be solved efficiently for graphs that have a ‘tree-like’ structure. This allows us to generate a large and conceptually diverse corpus of mathematically deep problems, each guaranteed to have an efficient solution, yet potentially being extraordinarily challenging to discover in practice.

The problems in FormulaOne are designed to serve as a ``new ARC'' (c.f., \cite{chollet2019measure, chollet2024arc, chollet2025arc}) for mathematical reasoning, demanding a synthesis of skills from topological and geometric insight, to knowledge of graph theory, and the need for precise implementation.
While ARC is designed to measure fluid intelligence by evaluating performance on tasks that are explicitly out-of-distribution (OOD) relative to the training examples, FormulaOne presents a challenge that is,
by design, entirely in-distribution. Every problem, from the simplest to the most complex, is generated from the same family: MSO logic on graphs. Thus our dataset consists entirely of algorithmic coding problems, a task on which frontier reasoning models should,
by rights, perform well. 
Nevertheless, even the best frontier reasoning models, which excel at human-level competitions
such as OpenAI's o3, completely fail on our dataset, achieving a stark $<$1\% success rate.

One of the primary characteristics of the problems appearing in FormulaOne, is the amount of reasoning required.
In \Cref{sec:appendix_mcg} we exemplify this by proving a solution, in full, to one such problem.
Astonishingly, there are no fewer than $15$ interdependent, highly complex mathematical reasoning steps, all intertwined in non-obvious ways.
We conjecture that this reasoning depth, typical of cutting edge real world research problems, 
is the main characteristic due to which frontier AI models ``flat-line" on FormulaOne.
This grim result highlights a pressing need for deeper reasoning environments and better benchmarks,
capturing increasing levels of complexity, and perhaps necessitating a more structured approach.
To support further research, we will release the full dataset along with a comprehensive evaluation framework.
We believe this provides a solid foundation to guide and measure future progress in advanced algorithmic reasoning.

Another key characteristic of FormulaOne is its profound connections to the frontier of theoretical computer science and central conjectures therein.
A prime example is the Strong Exponential Time Hypothesis (SETH), a foundational conjecture in fine-grained complexity.
Informally, SETH posits that the classic brute-force search algorithm for the Boolean Satisfiability (SAT) problem is
essentially optimal, meaning no algorithm can provide a significant exponential speedup.
The time complexity of many core graph problems, including several in our dataset, is believed to be optimal under SETH. 
That is, no algorithm can solve them faster than a particular known lower bound, parameterised 
by the tree-likeness of the input graph.
Therefore, if a powerful AI agent were able to discover a genuinely novel, faster algorithm for one of these hard tasks, it
would do more than just solve a puzzle; it would effectively refute a central hypothesis in theoretical computer science.

The semi-mechanistic nature of our problem generation provides the first steps towards an essentially unbounded source of high-depth algorithmic challenges, ideal for building next-generation environments for Reinforcement Learning with Verifiable Rewards (RLVR).
Existing RLVR benchmarks are often limited in one of two ways: they either feature problems with low conceptual depth,
such as school-level mathematics \cite{cobbe2021training, hendrycks2021measuring}, or very simple programming tasks \cite{austin2021program}, or the datasets are more challenging but static and limited in size~\cite{imajuku2025ale}. 
Our framework of problems derived from MSO logic addresses both points:
it can provide a virtually infinite stream of problems with profound mathematical depth,
for which the solutions are nevertheless automatically verifiable.
This combination of unbounded difficulty and guaranteed correctness is critical for training
agents to tackle genuinely open-ended scientific discovery.

\subsection{Related Work}

Recent years have seen a proliferation of benchmarks at the intersection of mathematics, coding, and reasoning. We group our discussion of this prior work into four main areas.

\paragraph{Algorithmic Coding Benchmarks.}
Several benchmarks have been developed to measure the algorithmic problem-solving abilities of AI systems. ALE-Bench \cite{imajuku2025ale} introduced over 40 hard optimisation tasks from AtCoder contests, focusing on iterative solution refinement. Others, like CodeElo \cite{quan2025codeelo} and LiveCodeBench \cite{jain2024livecodebench}, provide rigorous evaluation on live contest problems, assigning models human-comparable Elo ratings in authentic execution environments. However, top models are rapidly approaching the performance ceiling on these platforms; OpenAI's o3 model, for instance, has achieved an Elo rating above 2700 on Codeforces \cite{quan2025codeelo}, indicating this domain is becoming saturated. In contrast, FormulaOne provides a challenge that is harder to saturate, as its difficulty stems not from eclectic puzzle design but from the profound reasoning depth characteristic of substantive research problems.

\paragraph{Out-of-Distribution Reasoning.}
The Abstraction and Reasoning Corpus (ARC) \cite{chollet2019measure, chollet2024arc} assesses generalisation to explicitly out-of-distribution (OOD) tasks. It presents visual puzzles that require an agent to infer a transformation from a handful of pixel-grid examples with no linguistic hints. Our work takes a complementary direction, probing deep \emph{interpolation} within a domain --- algorithmic programming --- the `bread and butter' of modern reasoning models' training. The failure of these models on our benchmark is therefore particularly notable: it demonstrates that even within a familiar domain, they cannot yet perform the multi-step reasoning that is required in order to successfully tackle such problems.

\paragraph{Frontier Knowledge Benchmarks.}
Other benchmarks test the limits of broad, specialised knowledge. Humanity's Last Exam (HLE) \cite{phan2024hle} assembles thousands of graduate-level questions from dozens of academic disciplines, while FrontierMath \cite{glazer2024frontiermath} collects challenging research-level problems in mathematics. These benchmarks consist of fixed, static problem sets, whereas our approach allows for semi-automatic generation of problems, in an unbounded domain (MSO). Furthermore, our family of problems is incredibly well-suited to the creation of powerful RLVR environments, with a variety of granular reward signals available, which prior datasets lack.

\paragraph{AI Pushing the Frontier of Theory.}
Very recently, there has been a push for AI-assisted discovery of new algorithms. Systems like AlphaTensor \cite{fawzi2022alphatensor} and AlphaEvolve \cite{novikov2024alphaevolve} have used deep reinforcement learning to find faster algorithms for specific cases of matrix multiplication, while AlphaDev \cite{real2023alphadev} found marginal improvements for sorting routines. These projects are typically structured as closed-loop discovery processes focused on a handful of human-selected challenges. Critically, their discoveries, while impressive, are conceptually distant from improving theoretical bounds, for instance, making progress on the exponent of matrix multiplication, $\omega$. Our work contributes to this area by providing an open-ended suite of problems where discovering a faster-than-previously-known algorithm could have genuine theoretical consequences. Since many of our challenges are related to conjectures like the Strong Exponential Time Hypothesis (SETH), a truly novel solution would be of huge theoretical significance.
\pagebreak
\section{Dataset}
\label{sec:dataset}

\tikzstyle{vertex}=[circle, inner sep=2pt, minimum size=5pt]
\tikzstyle{filled_vertex}=[vertex, fill=otherblue, draw=black]
\tikzstyle{unfilled_vertex}=[vertex, draw=black, thick]

\newcommand{\GraphIS}{%
    \begin{tikzpicture}[scale=0.8]
        \node[filled_vertex]   (h1) at (-0.5, -0.5) {}; 
        \node[unfilled_vertex] (v1) at (0,0) {};
        \node[unfilled_vertex] (v2) at (1.5,0) {};
        \node[filled_vertex]   (h2) at (2.0, -0.5) {}; 
        \node[filled_vertex]   (v3) at (0.75, 1.2) {};  
        \draw (h1)--(v1)--(v2)--(h2);
        \draw (v1)--(v3)--(v2);
    \end{tikzpicture}%
}

\newcommand{\GraphNotIS}{%
    \begin{tikzpicture}[scale=0.8]
        \node[unfilled_vertex] (h1) at (-0.5, -0.5) {}; 
        \node[filled_vertex]   (v1) at (0,0) {};
        \node[unfilled_vertex] (v2) at (1.5,0) {};
        \node[filled_vertex] (h2) at (2.0, -0.5) {}; 
        \node[filled_vertex]   (v3) at (0.75, 1.2) {};  
        \draw (h1)--(v1)--(v2)--(h2);
        \draw (v1)--(v3)--(v2);
    \end{tikzpicture}%
}

\newcommand{\GraphSun}{%
    \begin{tikzpicture}[scale=0.9]
        \node[filled_vertex] (c1) at (90:0.7) {};
        \node[filled_vertex] (c2) at (162:0.7) {};
        \node[filled_vertex] (c3) at (234:0.7) {};
        \node[filled_vertex] (c4) at (306:0.7) {};
        \node[filled_vertex] (c5) at (18:0.7) {};
        \node[filled_vertex] (l1) at (90:1.3) {};
        \node[filled_vertex] (l2) at (162:1.3) {};
        \node[filled_vertex] (l3) at (234:1.3) {};
        \node[filled_vertex] (l4) at (306:1.3) {};
        \node[filled_vertex] (l5) at (18:1.3) {};
        \node[unfilled_vertex] (cv) at (0,0) {};
        \draw (c1)--(c2)--(c3)--(c4)--(c5)--(c1);
        \draw (c1)--(l1) (c2)--(l2) (c3)--(l3) (c4)--(l4) (c5)--(l5);
        \draw (cv)--(c1) (cv)--(c3);
    \end{tikzpicture}%
}

\newcommand{\GraphNotSun}{%
    \begin{tikzpicture}[scale=0.9]
        \node[filled_vertex] (c1) at (90:0.7) {};   
        \node[filled_vertex] (c2) at (162:0.7) {};
        \node[filled_vertex] (c3) at (234:0.7) {};
        \node[filled_vertex] (c4) at (306:0.7) {};
        \node[filled_vertex] (c5) at (18:0.7) {};
        \node[filled_vertex] (l1) at (90:1.3) {};   
        \node[filled_vertex] (l2) at (162:1.3) {};
        \node[filled_vertex] (l3) at (234:1.3) {};
        \node[filled_vertex] (l4) at (306:1.3) {};
        \node[unfilled_vertex] (l5) at (18:1.3) {};
        \node[filled_vertex] (cv) at (0,0) {};
        \draw (c1)--(c2)--(c3)--(c4)--(c5)--(c1);
        \draw (c1)--(l1) (c2)--(l2) (c3)--(l3) (c4)--(l4) (c5)--(l5);
        \draw (cv)--(c1) (cv)--(c3);
    \end{tikzpicture}%
}

\newcommand{\GraphDS}{%
    \begin{tikzpicture}[scale=0.9]
        \node[filled_vertex] (v1) at (0, 1) {};
        \node[unfilled_vertex] (v2) at (-0.9, 0.4) {};
        \node[unfilled_vertex]  (v3) at (0.9, 0.4) {};
        \node[unfilled_vertex]  (v4) at (-0.6, -1) {};
        \node[unfilled_vertex]  (v5) at (0.6, -1) {};
        \node[filled_vertex] (v6) at (0, -0.3) {};
        
        \draw (v1)--(v2) (v1)--(v3);
        \draw (v2)--(v6);
        \draw (v3)--(v6);
        \draw (v4)--(v6) (v5)--(v6);
    \end{tikzpicture}%
}

\newcommand{\GraphNotDS}{%
    \begin{tikzpicture}[scale=0.9]
        
        \node[unfilled_vertex]  (v1) at (0, 1) {};
        \node[filled_vertex] (v2) at (-0.9, 0.4) {};
        \node[unfilled_vertex]  (v3) at (0.9, 0.4) {};
        \node[unfilled_vertex]  (v4) at (-0.6, -1) {};
        \node[filled_vertex]  (v5) at (0.6, -1) {};
        \node[unfilled_vertex] (v6) at (0, -0.3) {};
        
        \draw (v1)--(v2) (v1)--(v3);
        \draw (v2)--(v6);
        \draw (v3)--(v6);
        \draw (v4)--(v6) (v5)--(v6);
    \end{tikzpicture}%
}

\begin{figure}[t]
    \hspace{-1.3cm}
    \begin{tikzpicture}[scale=0.85] 
        \matrix[
            matrix of nodes,
            nodes={anchor=center},
            column sep=5pt,
            row sep=5pt 
        ] (m) {
            \problemblock{Independent Set}{\GraphIS}{\GraphNotIS} & 
            \problemblock{Sun Graph}{\GraphSun}{\GraphNotSun} &
            \problemblock{Dominating Set}{\GraphDS}{\GraphNotDS} \\
        };
        
        \path (m-1-1.east) -- (m-1-2.west) coordinate[pos=0.5] (vmid1);
        
        \path (m-1-2.east) -- (m-1-3.west) coordinate[pos=0.5] (vmid2);

        \draw (vmid1 |- m.north) -- (vmid1 |- m.south);
        \draw (vmid2 |- m.north) -- (vmid2 |- m.south);

        \node[fit=(m), draw, thick, rounded corners=15pt, inner sep=0pt, color=gray] (frame) {};
    \end{tikzpicture}%
    \caption{Three problems, expressible by MSO logic, solvable by dynamic programming on tree-like graphs.
    In each, the left-hand-side showcases a subgraph exhibiting a given property,
    whereas the right-hand-side showcases one that does not. Filled vertices,
    marked blue, denote the subset over which the graph is induced.}
    \label{fig:three-examples}
\end{figure}
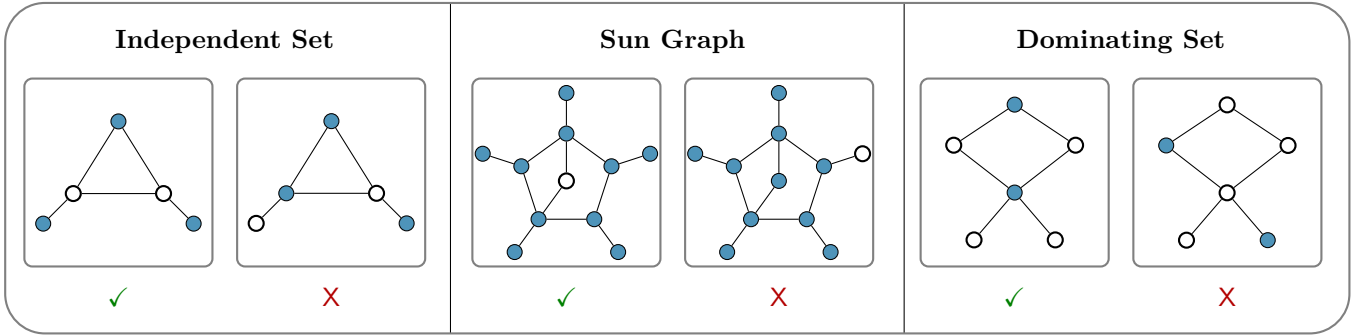

We introduce \textbf{FormulaOne}, a dataset consisting of a wide range of dynamic programming problems on graphs.
At the core of our work is the semi-automatic generation of a large corpus of
mathematically deep algorithmic problems, designed to gauge the command of abstract problem-solving,
multi-step combinatorial reasoning, and practical implementation.
Its theoretical foundation is rooted in an \textit{algorithmic meta-theorem}, due to Courcelle \cite{courcelle1990monadic}, which broadly states:
\begin{quoting}[indentfirst=true]
``For every \textit{sufficiently tree-like graph}, any problem definable in an expressive formal logic~--- Monadic Second-Order (MSO) logic --- can be solved by a dynamic programming algorithm that operates in time \textit{linear} in the order of the graph.''
\end{quoting}

Our problems all belong to a single underlying distribution: properties expressible through MSO logic.
The number of properties definable in Monadic Second Order Logic is, in principle, unlimited. 
Indeed, this family is incredibly diverse; covering both well-known classical problems,
such as $3$-colourability or counting maximal independent sets,
and entirely novel problems, including local invariants and topological structure (see \Cref{fig:three-examples}). 
Concretely, we generate problems of the following form.

\begin{tcolorbox}[
    arc=3mm,      
    colback=white,
    colframe=black,
    boxrule=1pt   
]
\vspace{0.1cm}\noindent\textbf{\underline{$\mathtt{Problem\ \#44}$}}\vspace{0.2cm}

\textbf{Input:} A tree-like graph $G = (V,E)$, a tree decomposition $\mathcal{T}$ of $G$, and a weight function $w: V \to \mathbb{N}$.\vspace{0.1cm}

\textbf{Objective:} Compute the sum of all weights of sets $S \subseteq V$ such that: \vspace{0.1cm}

\quad\quad\quad\quad\quad The graph $G[S]$, induced\footnote{The induced subgraph \( G[S] \) consists 
of all vertices in \( S \) and all edges of \( G \) with both endpoints in \( S \).} over $S$, does not
contain any cycle of length four.\vspace{0.15cm}

\textit{Notation: The weight of a set of vertices \( S \) is defined as \( w(S) \eqdef \sum_{v \in S} w(v) \).}\par

\end{tcolorbox}

In this context, a tree decomposition $\mathcal{T}$ is an auxiliary structure which enables for an efficient
dynamic programming algorithm to operate over the input graph. A detailed exposition is deferred to \Cref{sec:mso_is_hard}.

\subsection{Guiding Principles}

At the heart of FormulaOne lie three core principles.

\begin{enumerate}

\item \textbf{A ``New ARC'' for Mathematical Reasoning.} Our dataset consists of algorithmic coding problems requiring deep
mathematical reasoning.
Indeed, several of our problems directly relate to research papers at the forefront of
complexity theory \cite{kenig2024enumeration, fomin2012faster}.
The domain of algorithmic coding has become the core benchmark for measuring frontier reasoning models' progress.
For example, recently OpenAI's o3 reportedly ranked 175th on CodeForces \cite{el2025competitive}, relative to all human competitors.
However, while a human expert of that calibre should, by rights, be able to score highly on our problems,
we find that frontier models do not --- achieving a $<$1\% success rate on our hard problems (see \Cref{sec:results}).

\item \textbf{An Unbounded, Mathematically Deep Algorithmic Environment for RLVR.} The dynamic programming problems
generated through the use of Monadic Second-Order Logic provide an essentially unbounded source of algorithmic problems,
of incredible substance (see \Cref{sec:mso_is_hard}).
We believe our work provides the first steps towards enabling a truly high-depth environment for Reinforcement Learning with Verifiable Rewards (RLVR), finally moving research beyond existing datasets.

\item \textbf{Pushing the Boundaries of Complexity Theory.}
The complexity of problems expressible through MSO logic on tree-like graphs is intimately related to the Strong Exponential Time Hypothesis (c.f. \cite{lokshtanov2011known, cygan2015lower}), a central conjecture in the field of fine-grained 
complexity \cite{calabro2009complexity}.
Concretely, the best-known time complexity of a large portion of our dataset
is, in fact, optimal under SETH. Therefore, any \textit{significant} algorithmic progress on our dataset
could carry profound theoretical implications.
\end{enumerate}

\subsection{The dataset}
\label{subsect:dataset}

We introduce two datasets: a core dataset for benchmarking AI performance, and an auxiliary dataset for  research and evaluation.
\begin{enumerate}
\item \textbf{FormulaOne.} We introduce FormulaOne, a dataset of 120 challenging dynamic programming problems
that evaluate creativity, sophistication, and expert-level reasoning.

\item \textbf{FormulaOne-Warmup.}
To facilitate research and evaluation in this demanding setting, we also provide FormulaOne-Warmup, an auxiliary dataset containing 100 simpler problems.
\end{enumerate}

\noindent In \Cref{sec:results}, we provide an extensive evaluation of all top frontier reasoning models on our datasets.

\paragraph{Problem Formulation.}

Our problems are defined using MSO formulas.
An MSO formula is a statement in a formal logic used to express properties of graphs.
In this logic, one can define conditions on graphs, including by quantification over vertices, edges, and sets of vertices and edges
(`second-order'). Each problem in our dataset is defined by a \textit{unary open formula} in this logic --- a statement with a
single free variable that acts as a placeholder for an input set, which is either a set of vertices, or a set of edges. To help make matters concrete, let us consider a problem description in full.
\vspace{0.2cm}
\newtcolorbox{problemenv}{
    breakable, 
    colback=black!5, 
    colframe=black,
    boxrule=0.5pt, 
    arc=2mm,
    top=3mm,
    bottom=3mm,
    left=4mm,
    right=4mm
}

\begin{problemenv}
\scriptsize
\begin{verbatim}
## Description

Task Type: Weighted Model Count (CSP-wMC).
Explanation:
    You are given a set of elements, where each element has a weight. Your goal is to compute the sum of weights of all
    the subsets that satisfy the constraints given below, modulo 10^9 + 7.
    If there is no feasible set, the result of this task is -1.
Constraints Definition:
    In this question, the elements are the vertices of a graph, and the constraints are defined in graph-theoretic terms.
    A description of the graph family, the constraints, and the vertex weights is given below.
    The constraints:
        A set of vertices such that the graph induced over these vertices contains no induced square (C4). An induced
        C4 consists of 4 vertices that form a cycle of length 4, with no additional edges between them. 
    Input Graph:
        A graph whose tree-width is at most 3, and in whose 'nice tree decomposition' the JOIN nodes have width at most
        2 (i.e., every JOIN node contains at most 3 vertices).
        The following variables are used to represent graphs in this family:
            - `graph`: The graph.
            - `n`: The number of vertices in the graph.
            - `tree_decomposition`: The tree decomposition of the graph.
    Vertex Weights:
        Every vertex is assigned an integer weight. Each weight is between 1 and 10^5, inclusive. The list of weights
        is given in the variable `x_s`.

## Input

The input is composed of the following items, separated by newlines.

- n:
    An integer written on a single line.

- x_s:
    The first line contains a single integer representing the length of the list.
    The second line contains a list of space-separated integers.

- graph:
    An undirected graph over the vertex set of non-negative integers (starting at zero).
    The first line contains a single integer -- the number of vertices in the graph.
    The second line contains a single integer -- the number of edges in the graph.
    Every edge of the graph is then printed on a single line, as two space-separated integers.

- tree_decomposition:
    A tree decomposition of a graph.
    The first line contains a single integer -- the number of bags in the tree decomposition.
    The second line contains a single integer -- the number of edges in the tree decomposition.
    Then, for every bag in the decomposition, there is a single line of space-separated integers, representing the
    vertices in the bag (the i-th line corresponds to the i-th bag).
    Finally, every edge joining two bags of the tree decomposition is printed on a single line, as two space-separated
    integers, in ascending order. Note that the bags of the tree decomposition are indexed by non-negative integers,
    starting at zero.

## Output

A single integer: the sum of weights of all subsets that satisfy the constraints, modulo 10^9 + 7, or -1 if no such
subset exists.

## Constraints

The following constraints must all be satisfied:

n >= 4
n <= 94

## Time Limit

100.00 seconds per test.
\end{verbatim}
\end{problemenv}

\vspace{0.1cm}
As shown above, an input instance to a problem consists of a graph, a tree-decomposition\footnote{See section \ref{DP_on_tree_like_graphs}} of said graph,
and a weight function. If the formula's variable is a vertex set, weights are defined over the vertices;
otherwise, they are defined over the edges. 
The core objective for each problem in the dataset we release is Weighted Model Counting (WMC) ---
the goal is to compute the total weight of all subsets that satisfy the 
given formula, where the weight of a set $S$ is the sum of all weights of elements in $S$,
according to the given weight function.
Since this number may be large\footnote{Indeed, the truth set may be exponentially large in the order of the graph.},
we require only its remainder modulo the safe prime $10^9 + 7$.

\subsection{The Tip of the Iceberg -- Towards Scalable Problem Generation}
\label{subsect:tip_of_the_iceberg}
The current dataset, while comprehensive, represents just a \textit{fraction} of what is expressible
and solvable within the framework of MSO-based dynamic programming over graphs.
Let us mention, in passing, why.

\begin{itemize}
\item \textbf{Logical Expressiveness:} The vast majority of our formulas can be formulated in a subset of MSO logic, known as
$\text{MSO}_1$. This leaves out well-known broader classes, such as
$\text{MSO}_2$ and myriad other extensions of MSO, for which results similar to Courcelle's \cite{courcelle1990monadic} hold.

\item \textbf{Objective Variety:} We are releasing problems with WMC objectives only. However, the same logical framework
can be applied to optimisation problems (e.g., finding the minimum or maximum weight of a satisfying set) and other counting variants, which we do not include in this release. Such problems are qualitatively different and can be solved with state designs that are meaningfully different from their WMC counterparts (for instance, they do not require a `unique witness').

\item \textbf{Multi-Pass Traversal:} In our evaluation we allow for a single post-order traversal over the tree decomposition.
However, it is known (e.g., \cite{fichte2021dynasp2}) that allowing multiple passes or other traversal strategies may yield
substantial performance gains. Such approaches can simplify the required DP state and lead to alternative state representation and improvement to the overall running time.

\item \textbf{Base Graphs and Tree Decompositions:} Currently, we provide both a graph and a corresponding tree decomposition. This can be extended in two ways: by varying the families of input graphs, and by requiring the model to generate the tree decomposition itself.\footnote{In preliminary experiments on a small selection of problems, we found that prompting models to perform both the tree decomposition and the subsequent dynamic programming resulted in complete failure (0\% success rate).} We remark that while there is a linear-time algorithm for finding tree decompositions of tree-like graphs \cite{bodlaender1993linear}, the algorithm is not practical. Instead, for any \textit{particular} family of graphs (such as cactii, or $k$-trees), one can craft an \textit{efficient and practical} linear-time algorithm. Furthermore, there are non-trivial interactions between the structure of the input graph and the MSO formula. For example, if a problem canonically requires tracking information about odd-length cycles, this part of the DP state becomes redundant when the ambient graph is known to be bipartite. Recognising such properties allows for highly non-obvious state compression,
which is \textit{crucial} for efficient solutions.
\item \textbf{Beyond Treewidth:} This release is predicated on tree decompositions and the \textit{treewidth} parameter.
However, the underlying principle of Courcelle's theorem extends to many other graph parameters, such as \textit{pathwidth} \cite{robertson1983graph}, and \textit{clique-width} \cite{courcelle1993handle, wanke1994k}, and in certain cases, \textit{tree-depth} \cite{bodlaender1998rankings}.
\end{itemize}

\section{Complexity and Nuance in MSO-Based Dynamic Programming}
\label{sec:mso_is_hard}

\tikzstyle{vertex}=[circle, inner sep=2pt, minimum size=5pt]
\tikzstyle{filled_vertex}=[vertex, fill=otherblue, draw=black]
\tikzstyle{unfilled_vertex}=[vertex, draw=black, thick]
\tikzstyle{bag}=[draw=blue!80, thick, dashed, inner sep=5pt, rounded corners=5pt, draw opacity=0.7]
\tikzstyle{introduced_vertex}=[vertex, fill=nicegreen, draw=black]
\tikzstyle{forgotten_vertex}=[vertex, fill=nicered, draw=black]
\tikzstyle{join_vertex}=[vertex, fill=otherblue, draw=black, thick] 
\tikzstyle{comp_left}=[vertex, fill=orange, draw=black, thick] 
\tikzstyle{comp_right}=[vertex, fill=otherblue, draw=black, thick] 
\tikzset{
  half and half/.style n args={2}{
    shape=circle,
    draw,
    thick,
    inner sep=0pt,
    minimum size=6pt,
    path picture={
      \fill[#1] (path picture bounding box.south west) -- (path picture bounding box.north east) -- (path picture bounding box.north west) -- cycle;
      \fill[#2] (path picture bounding box.south west) -- (path picture bounding box.north east) -- (path picture bounding box.south east) -- cycle;
    }
  }
}

To properly contextualise the challenges and subtleties that arise in tackling the problems within our dataset,
we begin with a high-level ``crash-course'' overview of the methodology for dynamic programming on
tree decompositions. These topics are presented briefly, intuitively and mostly, visually. We then quickly proceed
to give a guided tour of the complexities inherent in our problems.

We remark that, despite their innocuous appearance, our MSO-based problems are complex and multifaceted. 
They demand combinatorial reasoning, geometric insights arising from the structure of the tree
decomposition, exact logic, and a careful implementation. Moreover, often these problems are riddled with 
\textit{subtle pitfalls}, all of which must be carefully sidestepped when implementing the dynamic programming
solution. The intent of this section is to showcase these intricacies. 

\subsection{Dynamic Programming on Tree-Like Graphs}\label{DP_on_tree_like_graphs}

\paragraph{Treewidth.} The methodology for solving these otherwise intractable problems relies heavily on the
``tree-likeness'' of the given graph, formally known as \textit{treewidth}
(c.f. \cite{bertele1973non, halin1976s, robertson1984graph, bodlaender1988dynamic}).
Loosely speaking, a graph of low treewidth locally resembles a tree. 
Any tree-like graph can be decomposed into a collection of small sets of vertices, ``bags'',
which are themselves organised into a tree structure. This construction, known as a \textit{tree decomposition},
provides a roadmap that allows an algorithm to process the entire graph by traversing this tree of bags.

Each bag acts as a small, \textit{local} ``window'' into the graph.
For the traversal process to be sound, every vertex is handled across a contiguous region
of the decomposition before it is ultimately processed and removed from view.
The tree decomposition is also carefully constructed to ensure that every vertex and edge of the graph is eventually processed.
With each step, the view changes only slightly and predictably: a single vertex might be introduced into the window,
forgotten from it, or two previously separate parts of the graph are joined (merged) within the view
(see \Cref{fig:bag-operations}).
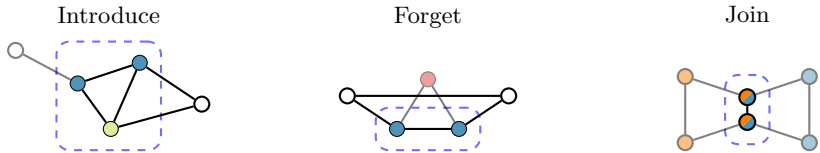
\begin{figure}[H]
\centering

\begin{subfigure}{0.25\textwidth}
    \centering
    \small Introduce \\ \vspace{2mm} 
    \begin{tikzpicture}[scale=0.55]
        \node[filled_vertex] (v1) at (0,0) {};
        \node[filled_vertex] (v2) at (1.5,0.5) {};
        \node[introduced_vertex] (v3) at (0.8, -1.1) {}; 
        \node[unfilled_vertex, opacity=0.5] (v4) at (-1.5, 0.8) {};
        \node[unfilled_vertex] (v5) at (3, -0.5) {};
        \begin{pgfonlayer}{background}
            \node[bag, fit=(v1) (v2) (v3)] {};
        \end{pgfonlayer}
        \draw[thick] (v1) -- (v2);
        \draw[thick] (v2) -- (v3);
        \draw[thick] (v1) -- (v3);
        \draw[thick, opacity=0.5] (v1) -- (v4);
        \draw[thick] (v2) -- (v5);
        \draw[thick] (v3) -- (v5);

    \end{tikzpicture}
\end{subfigure}
\begin{subfigure}{0.25\textwidth}
    \centering
    \small Forget \\ \vspace{5.45mm} 
    \begin{tikzpicture}[scale=0.55]
        \node[filled_vertex] (f1) at (0,0) {};
        \node[filled_vertex] (f2) at (1.5,0) {};
        \node[forgotten_vertex, opacity=0.5] (f3) at (0.75, 1.2) {}; 
        \node[unfilled_vertex] (f4) at (-1.2, 0.8) {};
        \node[unfilled_vertex] (f5) at (2.7, 0.8) {};
        \begin{pgfonlayer}{background}
            \node[bag, fit=(f1) (f2)] {};
        \end{pgfonlayer}
        \draw[thick] (f1) -- (f2);
        \draw[thick, opacity=0.5] (f1) -- (f3); 
        \draw[thick, opacity=0.5] (f2) -- (f3); 
        \draw[thick] (f1) -- (f4);
        \draw[thick] (f2) -- (f5);
        \draw[thick] (f4) -- (f5);
    \end{tikzpicture}
\end{subfigure}
\begin{subfigure}{0.25\textwidth}
    \centering
    \small Join \\ \vspace{5.72mm} 
    \begin{tikzpicture}[scale=0.55]
        \node[comp_left, opacity=0.5] (l1) at (-1.5, 0.8) {};
        \node[comp_left, opacity=0.5] (l2) at (-1.5, -0.8) {};
        \node[comp_right, opacity=0.5] (r1) at (1.5, 0.8) {};
        \node[comp_right, opacity=0.5] (r2) at (1.5, -0.8) {};
        \node[half and half={orange}{otherblue}] (s1) at (0, 0.3) {};
        \node[half and half={orange}{otherblue}] (s2) at (0, -0.3) {};
        \begin{pgfonlayer}{background}
            \node[bag, fit=(s1)(s2)] {};
        \end{pgfonlayer}
        \draw[thick, opacity=0.5] (l1) -- (s1) -- (r1);
        \draw[thick, opacity=0.5] (l2) -- (s2) -- (r2);
        \draw[thick, opacity=0.5] (l1) -- (l2);
        \draw[thick, opacity=0.5] (r1) -- (r2);
        \draw[thick] (s1) -- (s2);
    \end{tikzpicture}
\end{subfigure}

\caption{Local modifications to bags (dashed boxes). On the left, a vertex (green) is added (``introduced''), in the centre, a vertex (red) is removed (``forgotten''), and on the right, two previously processed graphs (orange, blue) meet at a bag (``joined''). Already forgotten vertices are semi-transparent. Vertices in the current bag, and those yet to have been processed, are opaque.}
\label{fig:bag-operations}
\end{figure}

The \emph{treewidth} of a graph simply corresponds to the size of the largest window needed for this process.
A low treewidth guarantees that our perspective is always limited to a small number of vertices at any given time,
making complex problems tractable.

\paragraph{Dynamic Programming.}
The algorithm that solves problems on graphs of low treewidth is a form of dynamic programming, 
wherein we traverse the tree decomposition of a given graph in post-order, considering a single bag at a time.
In this context, a partial solution refers to a solution to the problem at hand, restricted only to the portion of
the graph processed so far, i.e., the subgraph induced by the vertices appearing in the bags that are descendants of the 
current bag. Often, partial solutions are grouped within a bag according to their \textit{profile}, which acts as a summary
for the way in which they interact with the bag (see \Cref{fig:dp-snapshot}).

The key lies in what information we store for each bag --- the ``state''. 
The art of state design is in crafting a good profile that summarises the
properties of a \textit{partial solution}, as viewed from the vantage point of a bag. 
Such a state must be rich enough so as to allow it to be updated as we transition from one bag to the next --- whether introducing or removing a vertex,
or merging two views --- yet concise enough to be computationally tractable.
To design such a state, one must answer:
``What information do we need to know about the past and present, to make valid choices for the future?''.

\usetikzlibrary{positioning,fit,backgrounds,calc,matrix,arrows.meta}
\newcommand{\stateIN}[1]
{{\scriptsize\textcolor{darkgreen!90!black}{#1}}}
\newcommand{\stateOUT}[1]
{{\scriptsize\textcolor{nicered!90!black}{#1}}}
\tikzset{graph_vertex/.style={circle,draw=black,fill=white,inner sep=0pt,minimum size=15pt}}

\newcommand{\TDTraversal}{%
  \begin{tikzpicture}[scale=0.6, baseline=(bag_node.center), remember picture]

    \node[graph_vertex,   opacity=0.5] (a) at (-1,  0.0) {\small $a$};
    \node[graph_vertex, opacity=0.5] (b) at (1,  0.0) {\small $b$};
    \node[graph_vertex, opacity=0.5] (c) at (0, 1.0) {\small $c$};

    \node[graph_vertex] (d) at ( 0,  2.4) {\small $d$};
    \node[graph_vertex] (f) at ( 1.5, 2.0) {\small $f$};
    \node[graph_vertex]   (e) at ( 1.5,  3.4) {\small $e$};

    \node[graph_vertex] (g) at ( 2.3,  4.9) {\small $g$};
    \node[graph_vertex] (h) at ( 4.0,  4.9) {\small $h$};
    \node[graph_vertex] (i) at ( 3.3, 6.3) {\small $i$};

    \begin{pgfonlayer}{background}
        \node[bag, fit=(d) (e) (f)] (bag_node) {};
    \end{pgfonlayer}
    \node[anchor=north, font=\small] at (bag_node.north) [yshift=15pt] {Bag  $\mathcal{B}$};

    \draw[thick, opacity=0.5] (a) -- (c) (b) -- (c);
    \draw[thick, opacity=0.5] (c) -- (d) (c) -- (f);
    \draw[thick] (d) -- (e) (e) -- (f);
    \draw[thick] (e) -- (g) (g) -- (h) (g) -- (i);

    \coordinate (bagTop) at ($(bag_node.north east)+(-0.05, -0.05)$);
    \coordinate (bagBot) at ($(bag_node.south east)+(-0.05,  0.05)$);

    \node[draw, rounded corners, thick,
          inner sep=6pt, font=\scriptsize,
          align=center, anchor=west] (statebox)
          at ($(bag_node.east)+(3.6, 0)$) %
          {%
            \renewcommand{\arraystretch}{1.1}%
            \scalebox{0.9}{%
            \begin{tabular}{@{} l | c @{}}
            \quad\quad\ \textbf{Profile} & \textbf{Value} \\\hline
            $S_1 \cap \mathcal{B} = \{d, e, f\} $ & \dots \\
            $S_2 \cap \mathcal{B} = \{d, f\} $ & \dots \\
            $S_3 \cap \mathcal{B} = \{e, f\} $ & \dots \\
            \end{tabular}           
            }
          };

    \node[anchor=north, font=\small] at (statebox.north) [yshift=15pt] {DP State Table};

    \draw (bagTop)[opacity=0.7, dotted]  -- ($(statebox.north west)+(0,  0)$);
    \draw (bagBot)[opacity=0.7, dotted]  -- ($(statebox.south west)+(0, 0)$);

    \tikzset{bagcircle/.style={
        circle, draw=black, thick,
        minimum size=0.6cm,
        inner sep=0pt,
        font=\scriptsize,
        text width=1.7em,
        align=center
    }}
    
    \matrix (bagseq) [
        matrix of nodes,
        nodes={bagcircle}, 
        row sep=8pt,
        anchor=west
    ] at ($(statebox.east)+(3.0,0)$)
    {
      $efg$ \\
      $ef$ \\
      |[bagcircle, draw=blue!80, dashed]| $def$ \\ 
      |[opacity=0.5]| $df$ \\
      |[opacity=0.5]| $cdf$ \\
    };
    
    \draw[thick,opacity=0.5,-{Latex[length=4pt,width=4pt]}] (bagseq-5-1.north) -- (bagseq-4-1.south);
    \draw[thick,opacity=0.5,-{Latex[length=4pt,width=4pt]}] (bagseq-4-1.north) -- (bagseq-3-1.south);
    \draw[thick,-{Latex[length=4pt,width=4pt]}] (bagseq-3-1.north) -- (bagseq-2-1.south);
    \draw[thick,-{Latex[length=4pt,width=4pt]}] (bagseq-2-1.north) -- (bagseq-1-1.south);
    
    \node[font=\bfseries\scriptsize] at ($(bagseq-1-1.north)+(0, 0.5)$) {\vdots};
    \node[font=\bfseries\scriptsize, opacity=0.5] at ($(bagseq-5-1.south)+(0, -0.3)$) {\vdots};

  \end{tikzpicture}%
}
\begin{figure}[H]
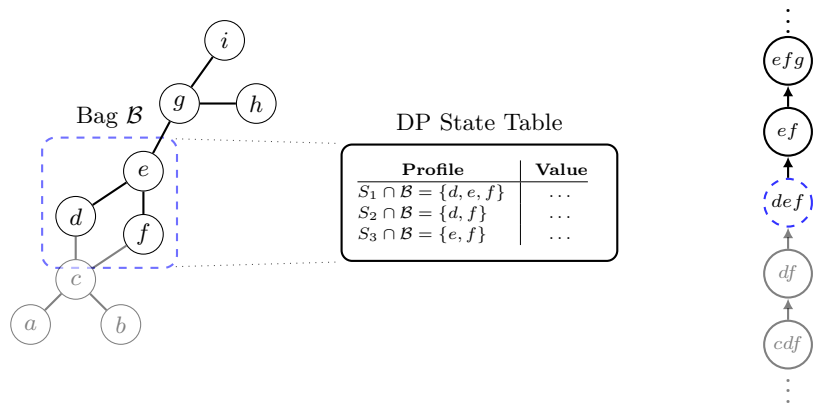

  \centering
  \TDTraversal
  \caption{Snapshot of a dynamic programming algorithm, traversing over a tree decomposition.
  The current bag is encircled in blue. 
    \textbf{Left:} The input graph $G$. Previously processed vertices and edges drawn semi-transparent.
    The current bag acts as a local window onto the vertices $d$, $e$ and $f$. \textbf{Centre:} The DP table associated with the current bag lists all possible \textit{profiles} — distinct ways in which \textit{partial solutions} (here, $S_1$, $S_2$ and $S_3$) can interact with the bag. Each row represents a grouping of partial solutions by their profile, and the corresponding value stores an aggregate measure over that group (e.g., maximum weight or total count).
    \textbf{Right:} The tree decomposition of $G$. The traversal proceeds bottom-up, aggregating information from subtrees until reaching the root, where a solution over the entire graph $G$ is extracted. In this case, the tree decomposition shown is linear (a path), that is, it contains no join bags (where subtrees merge).}
  \label{fig:dp-snapshot}
\end{figure}

\subsection{A very easy problem}

Before we jump into the deep end, let us dip our toes into the water through a very easy, classical problem.

\begin{tcolorbox}[
    arc=3mm,      
    colback=white,
    colframe=black,
    boxrule=1pt   
]
\vspace{0.1cm}\noindent\textbf{\underline{$\mathtt{Dominating\ Set}$}}\footnote{Hereafter, we omit the `Input' specification, as it is identical to the entry listed under \Cref{sec:dataset}.}\vspace{0.2cm}

\textbf{Objective:} Compute the sum of all weights of sets $S \subseteq V$ such that: \vspace{0.1cm}

\quad\quad\quad\quad\quad Every vertex outside the set is adjacent, in $G$, to at least one vertex in the set.\vspace{0.15cm}

\end{tcolorbox}

To tackle the dominating set problem we consider the different ways in which partial solutions, i.e., dominating
sets within the subgraph traversed so far, may interact with the vertices in the bag. The following DP state turns
out to be sufficient: for every vertex $v \in \mathcal{B}$ in the current bag $\mathcal{B}$, store one of the following
three configurations,

\begin{itemize}
    \item $\textcolor{otherblue!90!black}{\textsc{in}}$: The vertex is in the emerging dominating set.
    \item $\textcolor{darkgreen}{\textsc{out-dominated}}$: The vertex is \textit{not} in the emerging dominating set, but a previously seen vertex that is in the set, dominates it.
    \item $\textcolor{nicered}{\textsc{out-needs-domination}}$: The vertex is \textit{not} in the emerging dominating set, and is as of yet not dominated by any vertex chosen to participate in the set.
\end{itemize}

We stress that not every configuration of an emerging set may lead to a valid dominating set.
For instance, if a vertex is forgotten from the bag while in the status of $\textcolor{nicered}{\textsc{out-needs-domination}}$, 
then the emerging set corresponding to that state is \textit{not} dominating, and must be invalidated.
Notice that the aforementioned state aims to store the \textit{minimum amount of information} necessary in order to identify
the way in which the emerging dominating set interacts with the bag, and allows us to maintain the state as we traverse
the graph. For instance, we do not store irrelevant details, such as the identity of the vertices' dominators. 
We provide a full solution to the dominating set problem, written in Python, \href{https://github.com/double-ai/formulaone-dataset-release}{here}.

\newcommand{\DominatingSetBag}{ %
    \begin{tikzpicture}[scale=0.9, baseline=(bag_node_dom_bag_view.center), remember picture]

        \node[filled_vertex]   (a) at (-0.5,  0.0) {};  
        \node[unfilled_vertex] (b) at ( 0.0,  0.5) {};  
        \node[unfilled_vertex] (c) at (0.0, -0.5) {};  
        \node[unfilled_vertex] (d) at ( 0.5,  0.) {};  
        \node[filled_vertex, opacity=0.5]   (e) at ( 1.5, -0.4) {};  

        \begin{pgfonlayer}{background}
            \node[bag, fit=(a) (b) (c) (d)] (bag_node_dom_bag_view) {};
        \end{pgfonlayer}

        \draw[thick] (a) -- (b);
        \draw[thick, opacity=0.5] (d) -- (e);

        \node[anchor=north, font=\small] at (bag_node_dom_bag_view.north)
            [yshift=15pt] {Bag};
    \end{tikzpicture}%
}

\newcommand{\DominatingSetState}{%
  \begin{tikzpicture}[scale=0.9, baseline=(bag_node_dom_dp_state.center), remember picture]

    \node[filled_vertex]
          (a) at (-0.5,  0.0) {};

    \node[unfilled_vertex, draw=darkgreen, thick, fill=darkgreen, fill opacity=0.05]
          (b) at ( 0.0,  0.5) {};

    \node[unfilled_vertex, draw=nicered, thick, fill=nicered, fill opacity=0.05]
          (c) at (0.0, -0.5)  {};

    \node[unfilled_vertex, draw=darkgreen, fill=darkgreen, fill opacity=0.05]
          (d) at ( 0.5,  0.) {};

    \begin{pgfonlayer}{background}
      \node[bag, fit=(a)(b)(c)(d)] (bag_node_dom_dp_state) {};
    \end{pgfonlayer}

    \node[anchor=north, font=\small] at (bag_node_dom_dp_state.north)
        [yshift=15pt] {DP State};
        
    
    \draw[thick] (a.195) -- ++(195:4mm) coordinate (a_corner) -- ++(180:5mm) coordinate (a_end);
    \node[font=\scriptsize, anchor=south, color=otherblue!70!black] at ($(a_corner)!0.5!(a_end)$) {IN};

    \draw[thick] (b.15) -- ++(15:9mm) coordinate (b_corner) -- ++(0:1.7cm) coordinate (b_end);
    \node[font=\scriptsize, anchor=south, color=darkgreen!90!black] at ($(b_corner)!0.5!(b_end)$) {OUT-DOM};
    
    \draw[thick] (c.188) -- ++(188:10mm) coordinate (c_corner) -- ++(180:2.7cm) coordinate (c_end);
    \node[font=\scriptsize, anchor=south, color=nicered!90!black] at ($(c_corner)!0.5!(c_end)$) {OUT-NEEDS-DOM};

    \draw[thick] (d.310) -- ++(310:6mm) coordinate (d_corner) -- ++(0:1.7cm) coordinate (d_end);
    \node[font=\scriptsize, anchor=south, color=darkgreen!90!black] at ($(d_corner)!0.5!(d_end)$) {OUT-DOM};
    
  \end{tikzpicture}%
}

\begin{figure}[htbp]
    \centering

    \tikzstyle{bag}=[draw=blue!80, thick, dashed,
                     rounded corners=5pt, inner sep=5pt, draw opacity=0.7]

    \begin{tikzpicture}[remember picture, scale=0.9]
        \node (bag1)   at (0, 0)     {\DominatingSetBag};
        \node (state1) at (5.0cm, 0) {\DominatingSetState};

        \draw[overlay, gray, dashed, opacity=0.5] (bag_node_dom_bag_view.north east) -- (bag_node_dom_dp_state.north west);
        \draw[overlay, gray, dashed, opacity=0.5] (bag_node_dom_bag_view.south east) -- (bag_node_dom_dp_state.south west);

    \end{tikzpicture}

    \caption{Partial solutions as viewed through a bag, and the corresponding DP state. Vertices in the dominating set \( S \) are filled blue. \emph{Semi-transparent} vertices have already been ``forgotten''. The top vertex is dominated by a neighboring vertex in the set and currently in the bag. The right vertex is dominated by a previously processed vertex in the set and outside the bag. The bottom vertex is not yet dominated.
    }
    \label{fig:bag-and-state}
\end{figure}
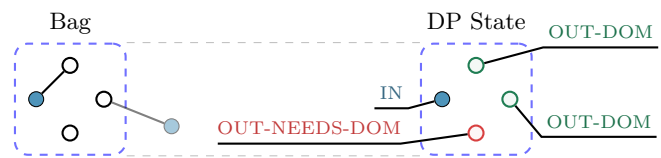

\subsection{A less innocent problem}
\label{subsec:a-less-innocent}

We can now begin to highlight the intricacies through a ``simple'' representative problem.

\begin{tcolorbox}[
    arc=3mm,      
    colback=white,
    colframe=black,
    boxrule=1pt   
]
\vspace{0.1cm}\noindent\textbf{\underline{$\mathtt{Connected\ with\ at\ least\ k\ vertices}$}}\vspace{0.2cm}

\textbf{Objective:} Compute the sum of all weights of sets $S \subseteq V$ such that: \vspace{0.1cm}

\quad\quad\quad\quad\quad The graph $G[S]$, induced over $S$, is connected and contains at least \( k = 4 \) vertices.

\end{tcolorbox}

On the face of it, the state one must store within the bag may appear relatively straightforward --- 
a counter for the size of the emerging set, and a running connectivity profile, i.e., a partitioning of the
vertices currently within the bag, indicating connectivity
with respect to the graph processed thus far.

\newcommand{\ConnAtLeastKBag}{ %
    \begin{tikzpicture}[scale=0.9, baseline=(bag_node_conn_bag_view.center), remember picture]

        \node[filled_vertex]    (a) at (0.000, 0.800) {};
        \node[filled_vertex] (b) at (0.761, 0.247) {};
        \node[filled_vertex] (c) at (0.470, -0.647) {};
        \node[filled_vertex] (d) at (-0.470, -0.647) {};
        \node[unfilled_vertex] (e) at (-0.761, 0.247) {};
        \node[filled_vertex, opacity=0.5]   (f) at ( 1.3, -0.4) {};  
        
        \begin{pgfonlayer}{background}
            \node[bag, fit=(a) (b) (c) (d) (e)] (bag_node_conn_bag_view) {};
        \end{pgfonlayer}

        \draw[thick] (a) -- (b);
        \draw[thick] (d) -- (e);
        \draw[thick, opacity=0.5] (b) -- (f);
        \draw[thick, opacity=0.5] (c) -- (f);

        \node[anchor=north, font=\small] at (bag_node_conn_bag_view.north)
            [yshift=15pt] {Bag};
    \end{tikzpicture}%
}

\newcommand{\ConnAtLeastKState}{%
  \begin{tikzpicture}[scale=0.9, baseline=(bag_node_conn_dp_state.center), remember picture]

    \node[filled_vertex]    (a_state) at (0.000, 0.800) {};
    \node[filled_vertex] (b_state) at (0.761, 0.247) {};
    \node[filled_vertex] (c_state) at (0.470, -0.647) {};
    \node[filled_vertex] (d_state) at (-0.470, -0.647) {};
    \node[unfilled_vertex] (e_state) at (-0.761, 0.247) {};

    \node[draw, rounded corners=3pt, inner sep=0pt, anchor=east, yshift=8pt, font=\small,
          label={[font=\small]above:Bag globals}] (globals_table) at (5, -0.47) {
        \begin{tabular}{l|l}
             \textbf{\scriptsize Key} & \textbf{\scriptsize Value} \\ \hline
             \scriptsize Set Size & \scriptsize $\ge 4$ \\
             \scriptsize Detached & \scriptsize False \\
        \end{tabular}
    };
    
    \begin{pgfonlayer}{background}

        \node[fit=(d_state), fill=nicegreen!95!black, opacity=0.5, rounded corners, inner sep=2pt] {};
        \node[enclosing_shape, color=lightred, opacity=0.7, fit=(a_state) (b_state) (c_state)] {};
        
        \node[bag, fit=(a_state) (b_state) (c_state) (d_state) (e_state) (globals_table)] (bag_node_conn_dp_state) {};
    \end{pgfonlayer}
    
    \node[anchor=north, font=\small] at (bag_node_conn_dp_state.north)
        [yshift=15pt] {DP State};
        
  \end{tikzpicture}%
}

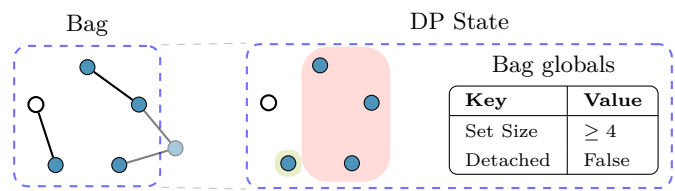
\begin{figure}[H]
    \centering

    \tikzstyle{bag}=[draw=blue!80, thick, dashed,
                     rounded corners=5pt, inner sep=5pt, draw opacity=0.7]

    \tikzstyle{enclosing_shape}=[fill=red, fill opacity=0.6, rounded corners=10pt, inner sep=4pt]

    \begin{tikzpicture}[remember picture, scale=0.8]
        \node (bag1)   at (0, 0)     {\ConnAtLeastKBag};
        \node (state1) at (6.0cm, 0) {\ConnAtLeastKState};

        \draw[overlay, gray, dashed, opacity=0.5] (bag_node_conn_bag_view.north east) -- (bag_node_conn_dp_state.north west);
        \draw[overlay, gray, dashed, opacity=0.5] (bag_node_conn_bag_view.south east) -- (bag_node_conn_dp_state.south west);

    \end{tikzpicture}

    \caption{A bag and its state. Vertices in the set \( S \) are filled blue. \emph{Semi-transparent} vertices have already been ``forgotten''. The connectivity partition has two classes (red, green) within the bag,
    in part owing to edges to vertices already forgotten, and the partial solution contains at least four vertices.
    Furthermore, there is no connected component of selected vertices, that has fully detached from the bag.
    }
    \label{fig:conn_at_least_k}
\end{figure}

So, how may one \textit{get the solution wrong}? Let us count the ways.\footnote{These listed errors were \textit{all} made by frontier reasoning models, including Google's Gemini 2.5, and OpenAI's o3.}
\paragraph{Premature invalidation of states.}
A natural mistake arises when prematurely discarding DP states. In this case, within the current bag, an emerging solution may appear to be disconnected --- but one must take into account the possibility that \textit{currently} disjoint components \textit{may merge} into a single connected component \textit{through yet-unseen vertices}, that will be processed at a later stage (i.e., introduced in future bags).
\paragraph{Failure to cap solution cardinality.}
Maintaining a minimal state profile requires that once the partial solution reaches the minimal size threshold, $k=4$, the DP
state no longer tracks \textit{exact} counts, but rather merely maintains that the threshold has been met. Failing to do so
results in a state-space explosion -- inflating the running time complexity from linear to cubic\footnote{The cubic runtime follows since the number of states per bag may now be linear in the order of the graph. Joining two such states naively requires quadratic time, and may happen $\Theta(n)$ times.}, due to redundantly tracking detailed counts.

\paragraph{Tracking external connected components.}
During the traversal of the tree decomposition, there may arise partial solutions that \textit{no longer intersect the current bag},
but nevertheless represent a valid connected subgraph. To handle these cases correctly, a single bit must track whether such an external connected
component exists, to be used to invalidate states in two situations: (a) If the intersection of the selected set and the bag is nonempty, we disqualify the state as this bag's vertices are forever disconnected from the external connected component. (b) During a join operation, if the two joined states both have an external connected component, joining them would mean joining two distinct connected components, resulting in a disconnected graph. This is a crucial detail that is easy to overlook.

\paragraph{Avoiding double-counting during joins.}
When merging states at a join node, one must carefully avoid double-counting vertices contained within the intersection (the bag itself). Naively summing the vertex counts from both states is incorrect; after all, vertices in the intersection must be counted exactly once. Additional subtle errors may occur when handling capped arithmetic, especially when transitioning between capped and uncapped counting schemes.

\paragraph{Careful merging of connectivity partitions.}
At the join nodes partial solutions from two distinct subtrees are merged. While it is essential that the subset of vertices selected for the emerging set within the intersection bag \textit{matches exactly} between merged states, their connectivity partitions \textit{might differ}. Correctly merging these partitions (by unifying any two sets that share a vertex between the two partitions) should be done efficiently,
for instance, by using a DSU data-structure.

\subsection{Complex Profiles}

Often, even thinking of a valid representation of a state is a challenging task. Here is one such case.

\begin{tcolorbox}[
    arc=3mm,      
    colback=white,
    colframe=black,
    boxrule=1pt   
]
\vspace{0.1cm}\noindent\textbf{\underline{$\mathtt{Cograph}$}}\vspace{0.2cm}

\textbf{Objective:} Compute the sum of all weights of sets $S \subseteq V$ such that: \vspace{0.1cm}

\quad\quad\quad\quad\quad The graph $G[S]$, induced over $S$, contains no induced path of length 3.\footnote{The length of a path is the number of edges it contains.}
\end{tcolorbox}

This problem illustrates a more geometric complexity in the state representation. We would like to preserve a state 
that guarantees that, within the emerging set, there is no induced path of length 3. What sort of state should we keep?
To adequately answer, we must first understand the different ways in which a path of length 3 may be \textit{perceived} through
a sequence of bags, as we traverse over a tree decomposition.

\begin{figure}[H]
\centering

\begin{subfigure}{0.25\textwidth}
    \centering
    \begin{tikzpicture}[scale=0.55]
        \node[unfilled_vertex, opacity=0.5] (v1) at (0,1) {};
        \node[filled_vertex] (v2) at (2,0) {};
        \node[filled_vertex] (v3) at (4, 0) {}; 
        \node[unfilled_vertex, opacity=0.5] (v4) at (6, 1) {};
        
        \begin{pgfonlayer}{background}
            \node[bag, fit=(v2) (v3)] {};
        \end{pgfonlayer}
        
        \draw[thick, opacity=0.5] (v1) -- (v2);
        \draw[thick] (v2) -- (v3);
        \draw[thick, opacity=0.5] (v3) -- (v4);

    \end{tikzpicture}
\end{subfigure}
\begin{subfigure}{0.25\textwidth}
    \centering
    \begin{tikzpicture}[scale=0.55]
        \node[filled_vertex] (f1) at (0,0) {};
        \node[filled_vertex] (f2) at (2,0) {};
        \node[unfilled_vertex, opacity=0.5] (f3) at (4, 1) {}; 
        \node[unfilled_vertex, opacity=0.5] (f4) at (6, 1) {};
        
        \begin{pgfonlayer}{background}
            \node[bag, fit=(f1) (f2)] {};
        \end{pgfonlayer}
        
        \draw[thick] (f1) -- (f2);
        \draw[thick, opacity=0.5] (f2) -- (f3); 
        \draw[thick, opacity=0.5] (f3) -- (f4); 
    \end{tikzpicture}
\end{subfigure}
\begin{subfigure}{0.25\textwidth}
    \centering
    \begin{tikzpicture}[scale=0.55]
        \node[filled_vertex] (u1) at (0,0) {};
        \node[unfilled_vertex, opacity=0.5] (u2) at (2,1) {};
        \node[unfilled_vertex, opacity=0.5] (u3) at (4, 1) {}; 
        \node[filled_vertex] (u4) at (6, 0) {};
        
        \begin{pgfonlayer}{background}
            \node[bag, fit=(u1) (u4)] {};
        \end{pgfonlayer}
        
        \draw[thick, opacity=0.5] (u1) -- (u2);
        \draw[thick, opacity=0.5] (u2) -- (u3); 
        \draw[thick, opacity=0.5] (u3) -- (u4); 
    \end{tikzpicture}
\end{subfigure}

\caption{A few different bags of cardinality $2$, encircled in blue, that are possible when traversing over $P_4$, a path with four vertices. Semi-transparent vertices had already been forgotten.}
\label{fig:p4_figure}
\end{figure}
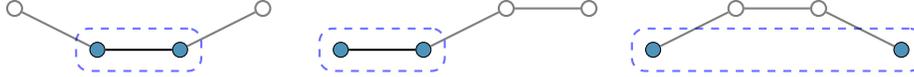

As it happens, one can indeed be convinced that no induced path of length 3 exists, by keeping track of the following:
\begin{enumerate}
    \item $\textcolor{nicered}{\textsc{out}}$: Indicating that the given vertex within the bag does not participate in the emerging set.
    \item $\textcolor{darkgreen}{\textsc{in-pair-and-have-common-forgotten-in-neigh}}$: For every \textit{pair} of vertices within the bag that are selected to participate in the emerging set ($\textsc{in}$), whether they share at least one common forgotten selected neighbor.
    \item $\textcolor{otherblue!90!black}{\textsc{in-and-forgotten-in-neigh}}$: For each selected vertex, record whether or not it has had a forgotten selected neighbor, so far.
    \item $\textcolor{tan!60!black}{\textsc{in-and-two-forgotten-in-neighs}}$: For each selected vertex, record whether or not there is a length-2 path leading up to it, consisting of two forgotten selected vertices.
\end{enumerate}

Preserving this state is no easier. At join nodes, complications emerge, as spurious induced paths of length three can arise
from combining subtrees. For instance, such paths can form when merging bags where a vertex from one subtree is an endpoint of a length-3 path, and simultaneously has a forgotten neighbor in the second subtree. Similarly, an induced path can emerge when there is a forgotten vertex connecting two vertices from one subtree, one of whom has an additional forgotten neighbor from the other subtree. And so forth \dots

\subsection{Mathematical Reasoning}

Sometimes, combining two problems can result in a problem exceeding (in elegance) the sum of its parts.
Consider the following example.

\begin{tcolorbox}[
    arc=3mm,      
    colback=white,
    colframe=black,
    boxrule=1pt   
]
\vspace{0.1cm}\noindent\textbf{\underline{$\mathtt{Bipartite\ Cograph}$}}\vspace{0.2cm}

\vspace{0.2cm}\textbf{Objective:} Compute the sum of weights of all subsets \( S \subseteq V \) such that the following holds:
\begin{enumerate}
    \item The induced subgraph \( G[S] \) contains no induced path of length 3.
    \item The induced subgraph \( G[S] \) is bipartite; that is, it contains no odd-length cycles.
\end{enumerate}
\end{tcolorbox}

This example highlights the value of mathematical knowledge. Mechanistically, one might combine separate solutions for tracking cographs\footnote{A cograph is a graph containing no induced path with four vertices and three edges.} and bipartite graphs, resulting in a large state and poor complexity. However, recognising the embedded combinational insight simplifies the solution significantly: a cograph is bipartite if and only if it is triangle-free.\footnote{This is illustrative of a more general theme; certain types of graph properties, for instance, those that are `induced-hereditary' such as the one shown in this problem, are characterised by a set of ``forbidden induced subgraphs'' (see \cite{Harary1969}). In this case, the forbidden graphs are $P_4$, a path with four vertices and three edges, and $C_3$, a triangle.}

\begin{lemma}
    \label{lem:bip_cograph}
    Let $G$ be a cograph. Then, $G$ is bipartite if and only if $G$ is triangle-free.
\end{lemma}
\begin{proof}
    Recall that a graph is bipartite if and only if it does not contain any odd-length cycle.
    Therefore, the first implication is immediate.
    In the opposite direction, let $G$ be a triangle-free cograph. Assume FSOC that $G$
    is not bipartite. Therefore, it contains (as a subgraph) an odd length cycle of length greater than three.
    Let $v_1 \sim v_2 \sim \dots \sim v_{2k + 1} \sim v_1$ be such an odd cycle, of minimum length, where $k > 1$.
    
    Consider the vertices $T = \{v_1, \dots, v_4\}$.
    Recall that $G$ is triangle-free, so except for the `path' joining those vertices,
    no two vertices may be joined by an edge, except perhaps for $v_1$ and $v_4$.
    If $v_1 \not\sim v_4$, then inducing over $T$ yields a copy of $P_4$, a contradiction.
    Otherwise, if $v_1 \sim v_4$, then the original cycle is split into a cycle of length four, and a cycle
    of length $2k - 1$, a contradiction.
\end{proof}

Another useful fact is that in a tree decomposition, every triangle (in fact, every clique) appears fully
within a bag (at some point). Thus, by relying on \Cref{lem:bip_cograph} one may greatly simplify the problem,
reducing the \textit{bipartiteness} to simply ensuring no bag introduces a triangle.
Indeed, the solution requires minimal additional overhead beyond that of the original $\mathtt{Cograph}$ problem.
\section{Evaluating Solutions}
\label{sec:evaluation}

A crucial aspect of our work is the ability to rigorously evaluate candidate solutions,
to any of our given algorithmic problems. To this end, we have developed a comprehensive
framework, allowing for both systematic generation of dynamic programming problems, and
verification of proposed solutions to them.

\subsection{Evaluation Environment}
\label{subsec:evaluation-extraction}

To focus the evaluation on the core algorithmic reasoning challenge, and to simplify the task presented to the
models to a large extent, we provide a purpose-built Python \href{https://github.com/double-ai/formulaone-dataset-release}{evaluation environment}.
This framework handles the input parsing, graph representation,
and the traversal of the tree decomposition. Consequently, the model's task is significantly streamlined: it need only
implement the \textit{logic} for the dynamic programming updates at each type of node in the tree decomposition.
Concretely, only the following five specific \textit{callback} functions must be implemented,

\begin{center}
\renewcommand{\arraystretch}{1.1}
\tikz\node[
    draw=black,
    thick,
    rounded corners=8pt,
    fill=lblue!80!white,
    inner xsep=1pt,
] (table) {
\setlength{\tabcolsep}{3pt} 
\newcolumntype{L}{>{\color{black}\ttfamily}l}
\newcolumntype{R}{>{\color{black}}l}
\begin{tabular}{@{ }L R@{ }}
    \multicolumn{1}{l}{\color{black}\textbf{\small Callback}} & \multicolumn{1}{l}{\color{black}\textbf{\small Description}} \\
    \cline{1-2}
    \rowcolors{1}{lightgray}{lightgray}
    leaf\_callback & Initialises the DP table at a leaf node of the tree decomposition (Base case). \\
    \rowcolor{colorBase}introduce\_callback & Handles the introduction of a vertex into a bag. \\
    forget\_callback & Handles the removal of a vertex from a bag. \\
    \rowcolor{colorBase}join\_callback & Handles the case in which two previously processed subgraphs meet at a bag. \\
    extract\_solution & Computes the final answer from the DP table at the root bag. \\
\end{tabular}
};
\end{center} 

The environment traverses through the tree decomposition in post-order, invoking the appropriate callback at each node.
To enforce that solutions are based purely on the principles of tree decomposition, the graph provided to each callback
is restricted to the subgraph induced by the vertices in the current bag.

\subsection{Problem Generation and Structure}

Our dynamic programming problems are generated by means of a domain-specific language (DSL).
For each such problem, the DSL produces:
\begin{enumerate}
    \item A human-readable description of the task (see \Cref{sec:dataset}).
    \item A \textit{verifier} $V$, mapping any input $(G, \mathcal{T}, w)$ to the integer answer to the problem at hand.
    \item A set of `tests', where each test consists of an \textit{input} -- which is a triplet $(G, \mathcal{T}, w)$ of a graph,
    a tree-decomposition, and a weight~function -- and optionally an \textit{output}, which is a single integer $x$.
\end{enumerate}

To ensure that the generated tests are tractable, our framework generates a complexity estimate for each task.
Every formula in our dataset is annotated with a specification for a valid (though not necessarily optimal) dynamic programming state.
This specification is constructed from a library of pre-defined, modular building blocks, which represent common patterns arising
in MSO-based dynamic programming, like state bits for tracking properties, or profiles for representing incidence (or higher-order) structures.

Our DSL leverages these state specifications to derive a concrete complexity bound for the problem,
which in turn informs the test generation. This model can also account for more advanced details,
such as rules for optimising computationally expensive operations like the `join' step (often reducing its
complexity from quadratic to near linear). We remark that the time limits for our tests are intentionally very generous;
we provide a performance overhead of up to $100\times$ over the expected runtime of the annotated solution,
ensuring that the evaluation prioritises algorithmic correctness over minor implementation-level inefficiencies.

\subsection{Generation of Tests}

To comprehensively test any proposed solution, it is crucial that in the generated tests, 
both the \textit{graph} and its \textit{tree-decomposition} be rich and interesting.
To this end, we employ a stochastic process; we sample the graph and a tree decomposition 
from a family of Markov chains, designed to explore the space of low-treewidth graphs.
This methodology provides us with exact control over the treewidth of the resulting graphs,
and other topological properties of the tree decomposition, thereby allowing us to ensure that
the tests for the problems remain tractable within an allotted timeframe.
Crucially, we remark that our sampling process ensures that (with high probability) all possible ``gadgets''
(small subgraphs) are \textit{present} in the tested graphs, and moreover, that they are
\textit{observed in all possible ways} through the local view of the bags.

\subsection{Types of Tests}

Our evaluation methodology consists of several key types of test suites,
each designed to probe a different aspect of a solution's validity.

\paragraph{Consistency.} A fundamental property of the evaluated dynamic programming algorithms
is that the final result should be \textit{invariant} to the choice of tree decomposition. That is,
for a fixed graph $G$, and weight function $w$, the output should not depend on the structure of the
tree decomposition $\mathcal{T}$. 
The consistency tests ensure that this property holds, by fixing a large graph $G$ and an arbitrary
weight function $w$, and enumerating over many different tree decompositions
of the same ambient graph, through a set of `perturbations' that successively modify the structure of an existing tree decomposition.

\paragraph{Correctness.} For sufficiently small input graphs, it is feasible to compute the expected output by brute force.
This is done by leveraging the aforementioned \textit{verifiers}. These tests serve as a direct verification of the
DP algorithm's logic. Our framework generates many small graphs and varied tree decompositions, and the output of the
candidate solution is compared against the known correct answer.

\paragraph{Efficiency.} These tests are designed to push a solution to its computational limits. The goal is to detect implementations where the DP state, or its handling, is not truly fixed-parameter linear (FPL). By systematically increasing
the difficulty of the tests, either by producing larger graphs, or graphs with certain extremal properties, we can identify performance bottlenecks that indicate a flawed or suboptimal implementation.

\paragraph{Sporadic (Exotic) Tests.} These includes using small ``universal'' graphs that are known to contain a wide range of subgraphs and structures, providing a dense set of features for a single test case. 

\section{Results}
\label{sec:results}

In this section we present the performance of frontier reasoning models on our datasets.

\subsection{Experimental Setup}
\label{sec:experimental_setup}

For our evaluation, we generate solutions from each model using a detailed prompt that provides the model with all necessary context.
Our goal is to help the model to the largest extent possible. 
The prompt, which can be found in full \href{https://github.com/double-ai/formulaone-dataset-release}{here}, consists of the following components:

\begin{itemize}
    \item A brief description of the dynamic programming setting, outlining the general approach for designing a state and the transition functions.
    \item The specific problem description, followed by the request to implement the four transition functions, and the extraction of the final answer from the root table (see \Cref{subsec:evaluation-extraction}).
    \item A set of basic guidelines, which instruct the model to avoid trying to circumvent the framework, state the expected fixed-parameter linear (FPL) runtime complexity, and encourage it to attempt a solution even for difficult problems.
    \item A utility class to query the graph induced by the vertices within a bag.
    \item \textbf{Three diverse few-shot example solutions}, included to aid the model, selected from areas where models were observed to struggle.
\end{itemize}

To score a model's response, the implemented callback functions are extracted from the completion and are integrated into the evaluation environment described in \Cref{subsec:evaluation-extraction}. They are then run against our test suites to produce a final score for the solution (see \Cref{sec:evaluation}).

\subsection{Evaluation Results}

We evaluated the four leading reasoning models on our dataset: o3 high (OpenAI), o3-Pro high (OpenAI), Gemini 2.5 Pro (Google DeepMind), and Grok 4 Heavy (xAI). In this evaluation, we provided each model with the maximum number of reasoning tokens, the maximum number of output tokens, and used the highest level of ``reasoning'' available.\footnote{Grok 4 Heavy was also provided with internet access, whereas other models were not due to technical constraints.} The models' only task is to implement the five necessary functions outlined in \Cref{sec:evaluation}.
To ensure comparable inference costs, we sampled ten \textit{independent} completions \textit{per problem} ($@10$) for o3 and Gemini 2.5 Pro, and a single completion ($@1$) for o3-Pro and Grok 4 Heavy (the latter two models already aggregate multiple samples internally).
A model is considered to have succeeded in solving a problem $@k$, if \textit{any} of its $k$ attempts has been successful.

The success rate on the FormulaOne dataset is remarkably low, with Grok 4 Heavy solving none of the problems, and Gemini, o3 and o3-Pro each solving only one problem out of 120. This poor performance is particularly notable given the substantial assistance provided: models received a diverse few-shot prompt, and the framework handled all input parsing, graph representation, and tree decomposition construction and traversal on their behalf. This indicates that the gap between current model capabilities and the reasoning required for these problems is fundamental and extends beyond prompt engineering.
We refer the reader to \Cref{subsect:central_failures}, where we provide a brief analysis of the failure
modes observed in our evaluation of these frontier models. 

\begin{figure}[H]
    \centering
    \includegraphics[width=0.72\textwidth]{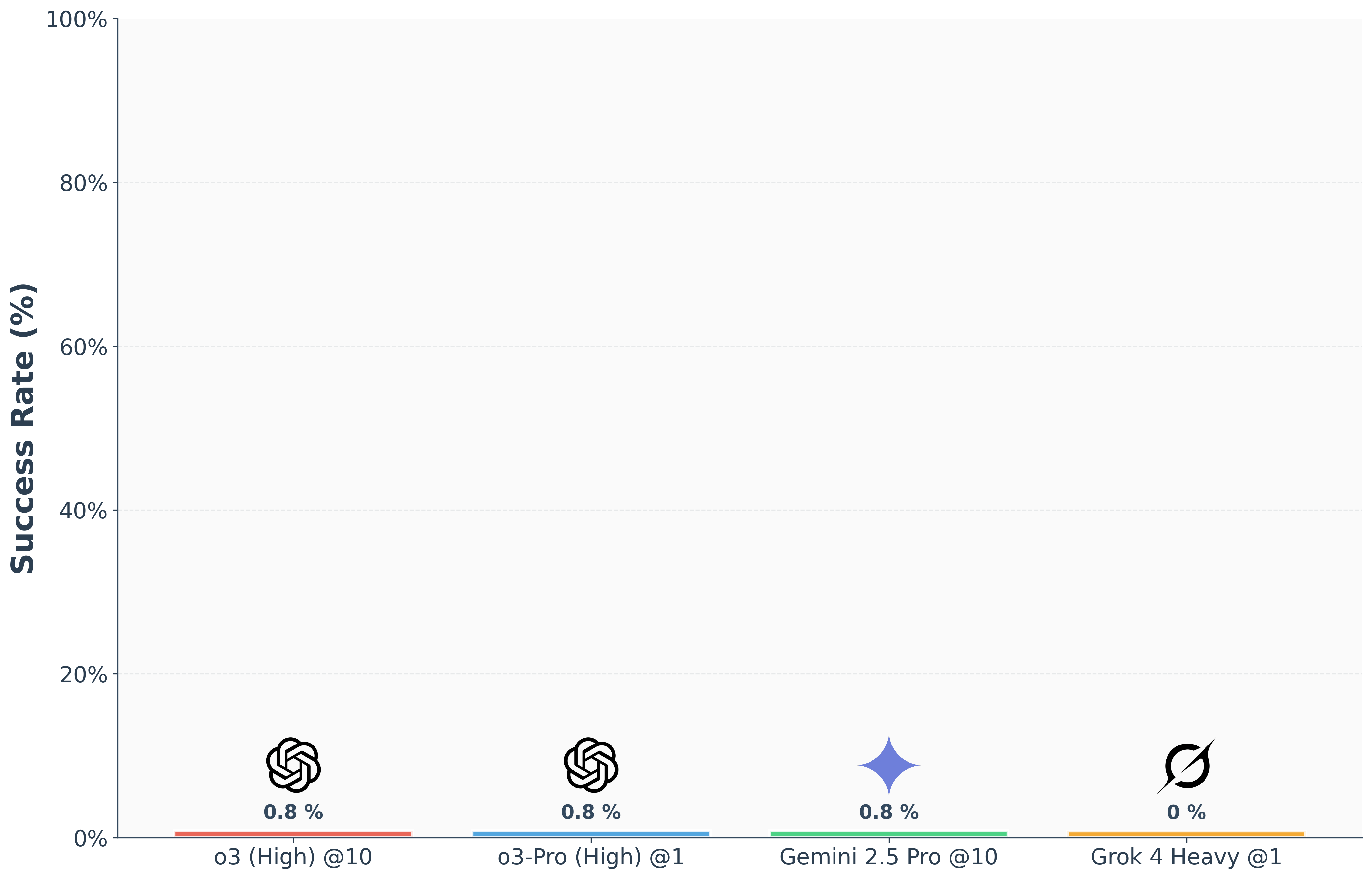}
    \caption{Performance of top frontier models on the FormulaOne dataset.}
    \label{fig:model_performance}
\end{figure}

The performance of reasoning models on FormulaOne-Warmup -- our dataset of \textit{entry-level} MSO-based problems -- was substantially better.
We believe these results clearly delineates the complexity spectrum within problems derived from MSO logic on graphs -- ranging from accessible entry-level reasoning tasks, to challenges that genuinely test the current upper bounds of algorithmic and logical reasoning capabilities in frontier models. We believe that FormulaOne-Warmup provides a valuable benchmark for the broader research community,
including those working with smaller or open-source models, and can serve as a training ground for developing more advanced algorithmic reasoning capabilities.

\begin{figure}[H]
    \centering
    \includegraphics[width=0.72\textwidth]{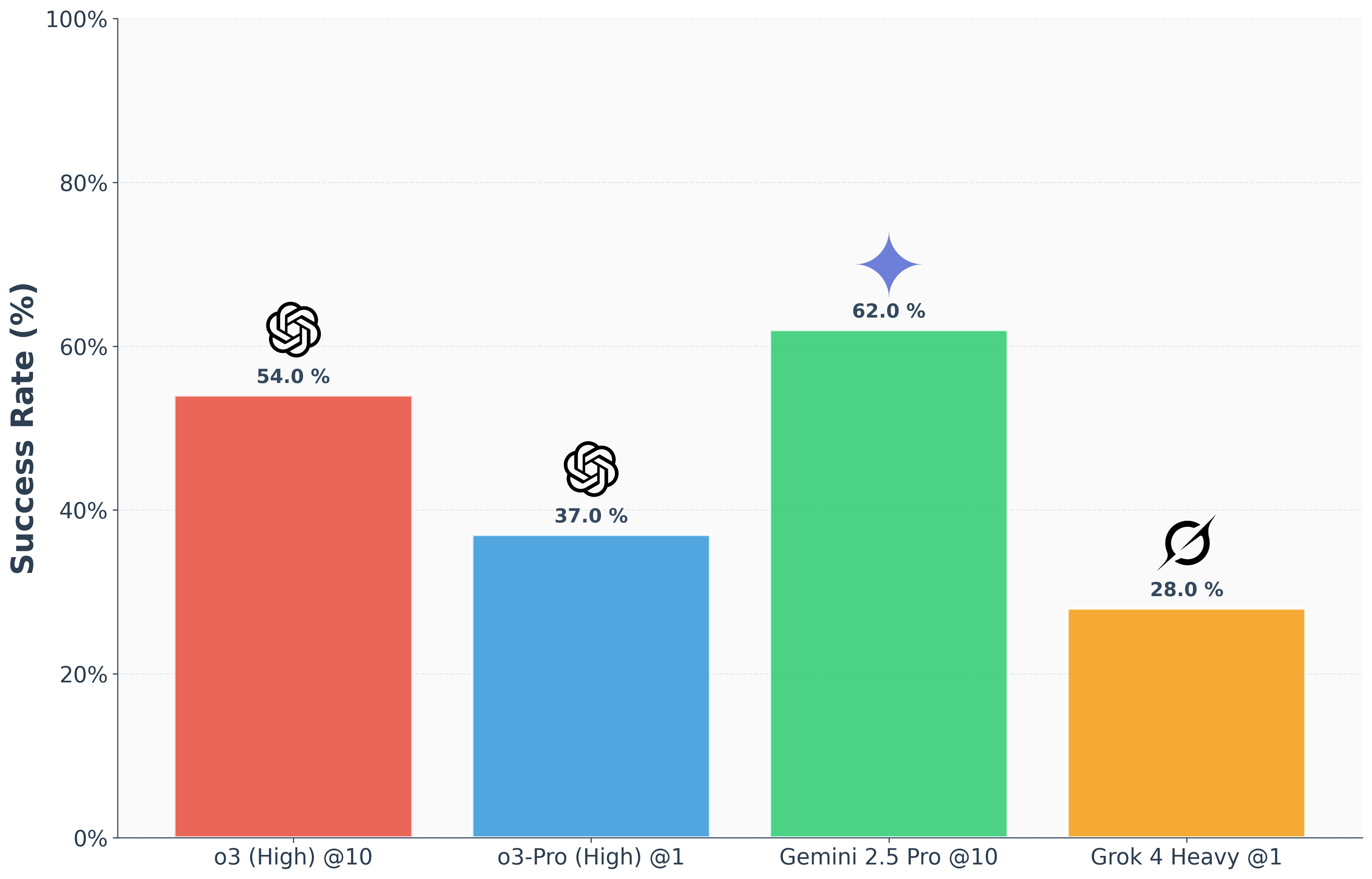}
    \caption{Performance of top frontier models on the FormulaOne-Warmup dataset. There are $100$ problems in the dataset. Performance is measured in terms of number of problems solved, where a problem is solved $@k$ if \textit{any} of the $k$ solutions is correct.}
    \label{fig:model_performance}
\end{figure}

\subsection{Further Analysis}

To better understand model performance and failure modes, each problem in our dataset is annotated according to specific algorithmic skills and state-design techniques required for its solution. Each problem can be assigned multiple labels (and indeed, harder problems often are). This categorisation allows for a fine-grained analysis of the strengths and weaknesses exhibited by the evaluated models.

\begin{center}
\begin{tcolorbox}[
    breakable, 
    colback=colorLight, 
    colframe=black, 
    boxrule=1pt, 
    arc=8pt, 
    boxsep=0pt, 
    left=0pt, right=0pt, top=0pt, bottom=0pt, 
]
\small
\renewcommand{\arraystretch}{1.4}
\begin{longtable}{@{} l @{} p{13.3cm} @{}}
    
    \multicolumn{1}{l}{\color{black}\textbf{\small Label}} &
    \multicolumn{1}{l}{\color{black}\textbf{\small Description}} \\
    \midrule
    \endfirsthead
    
    \multicolumn{2}{c}{\small\textit{Table \thetable, continued from previous page}} \\
    \midrule
    \multicolumn{1}{l}{\color{black}\textbf{\small Label}} &
    \multicolumn{1}{l}{\color{black}\textbf{\small Description}} \\
    \midrule
    \endhead
    
    \midrule
    \multicolumn{2}{r}{\small\textit{Continued on next page}} \\
    \endfoot
    
    \bottomrule
    \endlastfoot

    $\mathtt{ADJACENCY}$ & Requiring neighborhood properties such as degrees, domination, and independence. \\
    
    \rowcolor{colorBase}
    $\mathtt{COMPOUND}$ & Formulas constructed recursively. \\
    
    $\mathtt{CONNECTIVITY\ \ }$ & Concerning connectivity, including components, cuts, minors, and cycle detection. \\

    \rowcolor{colorBase}
    $\mathtt{DISTANCE}$ & Relies on distance information in the ambient graph. \\
    
    $\mathtt{EPSILON}$ & ``Almost'' properties, where a property holds after omitting or adding a single vertex. \\
    
    \rowcolor{colorBase}
    $\mathtt{EXTREMAL}$ & Members of a property, that are also maximal or minimal w.r.t inclusion. \\
    
    $\mathtt{GADGETS}$ & Aiming to find fixed subgraphs, such as triangles or cliques, within the selected set. \\
    
    \rowcolor{colorBase}
    $\mathtt{GLOBAL}$ & Pertains to properties that require tracking global information across the entire graph. \\
    
    $\mathtt{GRAPH}$-$\mathtt{THEORY}$ & Requires non-trivial theorems of graph-theory. \\
    
    \rowcolor{colorBase}
    $\mathtt{INTERFACE}$ & Tracks how selected vertices inside a bag interact with those outside the bag. \\
    
    $\mathtt{LOGIC}$ & Having a complex logical structure, requiring analysis of quantifiers and negations. \\
    
    \rowcolor{colorBase}
    $\mathtt{LOOKAHEAD}$ & Must utilize states that encode information about future constraints or dependencies. \\
    
    $\mathtt{MODULAR}$ & Compares or analyzes the lattice structure of vertex neighborhoods. \\
    
    \rowcolor{colorBase}
    $\mathtt{TOPOLOGY}$ & Reasons about global topological structure. \\

\end{longtable}
\end{tcolorbox}
\captionof{table}{Annotations of problems in our dataset. Labels correspond to algorithmic techniques and elements of state-design arising in FormulaOne problems.}
\label{tab:problem_categories}
\end{center}

Using the categorisation provided in \Cref{tab:problem_categories}, we provide a more fine-grained breakdown of the models' performance, 
for each such category \textit{in isolation}.

\begin{figure}[H]
\centering
\begin{tikzpicture}

\pgfplotstableread[col sep=comma]{
tag,gemini-pro,o3-batch,o3-pro,grok-4-heavy,total
GLOBAL,57.69,57.69,38.46,34.78,26
ADJACENCY,42.31,40.38,34.62,22.92,52
CONNECTIVITY,40.91,34.09,20.45,19.05,44
LOGIC,39.02,36.59,21.95,13.16,41
GADGETS,25.53,27.66,17.02,13.95,47
EXTREMAL,28.57,21.43,7.14,8.33,14
INTERFACE,27.27,13.64,4.55,14.29,22
EPSILON,25.00,0.00,0.00,25.00,4
GRAPH-THEORY,14.29,14.29,19.05,0.00,21
DISTANCE,9.52,4.76,4.76,4.76,21
TOPOLOGY,0.00,4.55,9.09,4.55,22
COMPOUND,12.82,5.13,0.00,0.00,39
LOOKAHEAD,0.00,2.17,4.35,0.00,46
MODULAR,0.00,0.00,0.00,0.00,10
}\mydata

\begin{axis}[
    ybar,
    bar width=5pt,
    width=\textwidth,
    height=8cm,
    enlarge x limits=0.05,
    ylabel={Success Rate (\%)},
    symbolic x coords={GLOBAL, ADJACENCY, CONNECTIVITY, LOGIC, GADGETS, EXTREMAL, INTERFACE, GRAPH-THEORY, EPSILON, COMPOUND, DISTANCE, TOPOLOGY, LOOKAHEAD, MODULAR},
    xtick=data,
    xticklabel={
    \pgfplotstablegetelem{\ticknum}{total}\of\mydata
    \strut\texttt{\tick}\ (\pgfplotsretval)
    },
    xticklabel style={rotate=45, anchor=east, font=\small},
    ymin=0,
    ymax=65,
    ymajorgrids=true,
    grid style=dashed,
    legend style={at={(0.5,-0.45)}, anchor=north, legend columns=-1},
    legend image code/.code={
        \fill[#1] (0cm,-0.1cm) rectangle (0.2cm,0.1cm);
    },
    nodes near coords,
    nodes near coords style={font=\scriptsize\bfseries, rotate=90, anchor=west},
    nodes near coords align={vertical},
]

\addplot[draw=plotred, fill=plotred] table[x=tag, y=o3-batch]{\mydata};
\addplot[draw=plotblue, fill=plotblue] table[x=tag, y=o3-pro]{\mydata};
\addplot[draw=plotgreen, fill=plotgreen] table[x=tag, y=gemini-pro]{\mydata};
\addplot[draw=plotyellow, fill=plotyellow] table[x=tag, y=grok-4-heavy]{\mydata};

\legend{o3 (High) @10\ \ , o3-Pro (High) @1\ \ ,Gemini 2.5 Pro @10\ \ , Grok 4 Heavy @1}

\end{axis}
\end{tikzpicture}
\caption{Model Performance by problem label on the \emph{combined} FormulaOne and FormulaOne-Warmup datasets.
Problem counts with the given category label are shown next to each label in parentheses.
The success rate is the ratio of problems solved, having the given label.}
\label{fig:model_performance_per_tag}
\end{figure}
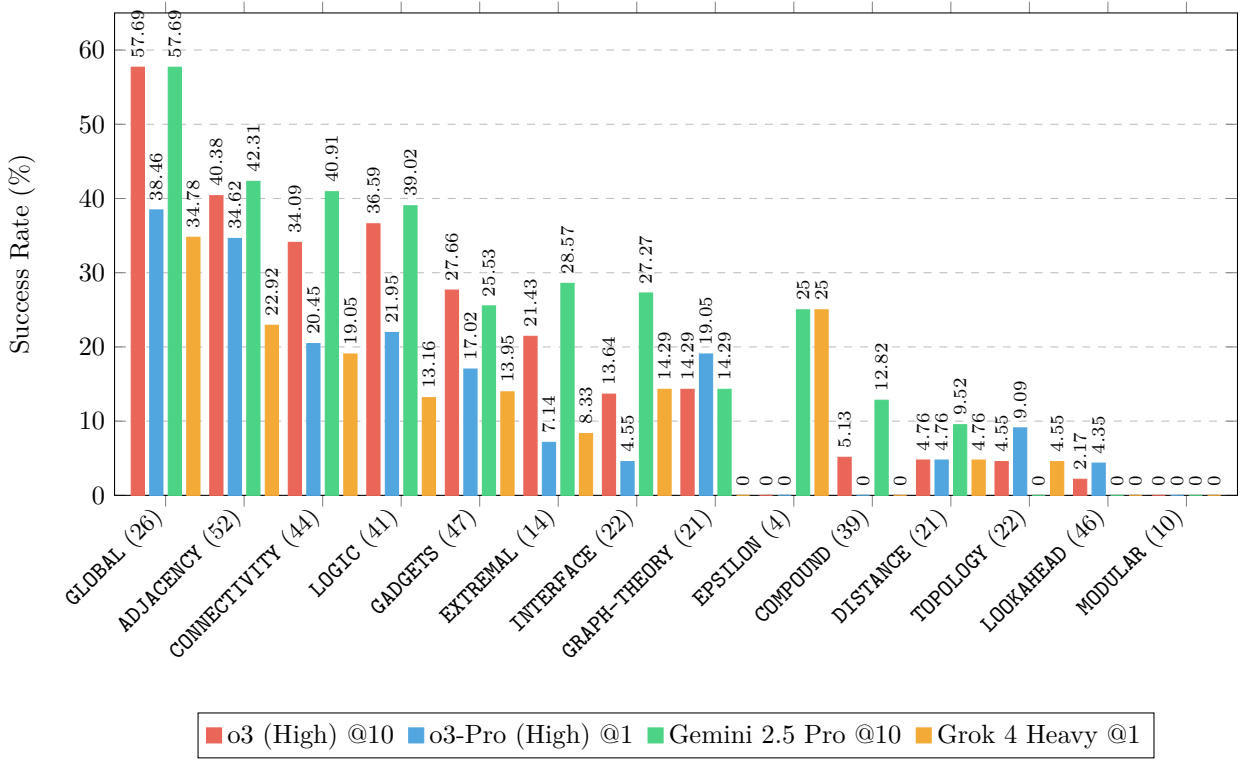

\subsection{Observed Central Failure Modes}
\label{subsect:central_failures}

Below, we provide a brief analysis of several recurrent failure modes observed during our evaluation,
highlighting specific areas where frontier reasoning models consistently struggle, despite extensive support.

\begin{enumerate}
\item{\textbf{Premature Finalisation.}}
Often models make irreversible decisions about a forgotten vertex based on the current,
non-final properties of its neighbors that remain in the bag. The failure is one of foresight;
as the model does not account for how these neighbors might change critically in the future.

\item{\textbf{Incomplete Geometric Reasoning.}}
A particular weak point of models is their inability to account for all the ways in which a small,
fixed subgraph ``pattern'' can be formed by combining vertices across different parts of the tree decomposition.
The states, or the transition logic, fails to cover all possible geometric configurations of the pattern's vertices with respect to the bags (see Figure \ref{fig:p4_figure} for an illustration of the kind of complications that may arise in such cases.)

\item{\textbf{Local-to-Global Errors.}}
This failure happens when the model successfully enforces local rules,
but fails to assemble them into the correct \textit{global} structure.
In other words, it builds an object that satisfies local constraints but violates the overall definition of the target structure.

\item{\textbf{Non-Canonical State Representation.}}
In counting problems, one must make sure to count every set exactly once. 
Often, the selected dynamic programming state includes auxiliary ``witness'' information,
that distinguishes between states that \textit{should} be considered identical,
due to a failure to define a \textit{canonical} representative for each state.
When unmitigated (for instance, by inclusion-exclusion), this leads to \textit{overcounting}.
    
\end{enumerate}
\section{Discussion}
\label{sec:discussion}    

While frontier AI models achieve high ratings in top human level competitive programming,
they fail on more challenging algorithmic challenges, such as the `hard' problems in our FormulaOne. This serves to highlight that current benchmarks,
which often rely on problems solvable by human experts, are insufficient for measuring the deep algorithmic and
combinatorial reasoning required for complex, real-world research tasks.  Given the depth of reasoning these problems demand,
future progress may depend on incorporating more principled approaches, such as systematic search, rather than relying solely
on the emergent capabilities of current models.

Another of our contributions is the harnessing of Monadic Second-Order Logic on graphs, in order to make the first steps
towards a principled method of creating a virtually unlimited suite of hard yet solvable problems. As such, we believe our dataset,
and its extensions, can serve as a high-depth environment for Reinforcement Learning with Verifiable Rewards (RLVR), moving beyond existing static datasets. 
We remark that the current dataset, while substantial, represents only a fraction of the challenges that can be generated (see \Cref{subsect:tip_of_the_iceberg}). 

Lastly, our work has a deep connection to theoretical computer science. The difficulty of many problems in our dataset
is linked to the Strong Exponential Time Hypothesis (SETH), a central conjecture in fine-grained complexity. This implies that
any significant algorithmic progress on these problems, whether by human or (hopefully) by machine, could have tangible theoretical implications, pushing at the frontiers of complexity theory.

\appendix
\section{Algorithm for Maximal Cluster Graph}
\label{sec:appendix_mcg}

Some problems are incredibly hard to get right from start to finish.
To illustrate why this is the case, let us consider one such problem, in full.

\begin{tcolorbox}[
    arc=3mm,      
    colback=white,
    colframe=black,
    boxrule=1pt   
]
\vspace{0.1cm}\noindent\textbf{\underline{$\mathtt{Maximal\ Cluster\ Graph}$}}\vspace{0.2cm}

\textbf{Objective:} Compute the sum of all weights of sets $S \subseteq V$ such that: \vspace{0.1cm}

\quad\quad\quad\quad\quad The graph $G[S]$ is a cluster graph\footnote{A cluster graph is a disjoint union of cliques.} that is maximal with respect to inclusion.
\end{tcolorbox}

The problem asks us to sum the weights of the sets $S$ for which $G[S]$ is a cluster graph,
\textit{and} there exists no superset $T \supseteq S$ such that $G[T]$ is a cluster graph.
We start with the following observation:
since the cluster graph property is downwards closed (with respect to inclusion),
maximality reduces to \textit{local maximality} -- that is, 
for any vertex $v \in V \setminus S$, the induced subgraph $G[S \cup \{v\}]$ is \textit{not} a cluster graph.

\subsection{Characterization}

A key observation is that a graph is a cluster graph if and only if it does not contain an induced path of length two (i.e., a copy of $P_3$). 
By definition a $P_3$ is a set of three vertices $\{u, v, w\}$ where the only edges are $u \sim v \sim w$.
Therefore, the problem is equivalent to finding a maximal vertex set $S \subseteq V$ such that the induced subgraph $G[S]$ is $P_3$-free. The maximality condition means that for any vertex $u \in V \setminus S$, adding $u$ to $S$ must create at least one induced $P_3$ in $G[S \cup \{u\}]$.

\subsection{Dynamic Programming Formulation}

As always, we solve the $\mathtt{Maximal\ Cluster\ Graph}$ problem through dynamic programming on a tree decomposition of the input graph $G$.
For each vertex $t$ in the tree decomposition with bag $X_t$, we compute a table $A_t$, which stores information about valid partial solutions
in the subgraph induced by the vertices processed so far, $V_t$.
A state in $A_t$ is defined by a \textit{profile} $\lambda$ for the vertices in the bag $X_t$. The table entry $A_t(\lambda)$ is a boolean flag,
indicating whether this profile can be extended from a valid partial solution on $V_t$.

\vspace{0.1cm}The profile $\lambda$ for $X_t$ consists of the following information for each vertex $v \in X_t$:

\begin{itemize}
    \item \textbf{Status $\mathtt{s}(v) \in \{\mathtt{selected}, \mathtt{not}\text{-}\mathtt{selected}\}$}: This indicates whether $v$ is part of the solution set $S$.
    \item \textbf{For each \textit{selected} vertex $v$}:
    \begin{itemize}
        \item \textbf{Forgotten Neighbour bit $\mathtt{FN}(v) \in \{0, 1\}$}, where $\mathtt{FN}(v)=1$ if and only if $v$ is adjacent to a \textit{selected} vertex in $S \cap (V_t \setminus X_t)$.
        \item \textbf{Obligation bit $\mathtt{O}(v) \in \{0, 1\}$}, where $\mathtt{O}(v)=1$ if and only if $v$ is required to become adjacent to a new \textit{selected} vertex in $S \setminus V_t$, to satisfy a maximality condition. 
    \end{itemize}

    \item \textbf{For each \textit{non-selected} vertex $u$}:
    \begin{itemize}
        \item \textbf{Safety bit $\mathtt{SAFE}(u) \in \{0, 1\}$}, where $\mathtt{SAFE}(u)=1$ if and only if the maximality condition for $u$
        is already satisfied (i.e., adding $u$ to $S \cap V_t$ would create an induced $P_3$).
    \end{itemize}
\end{itemize}

\subsection{DP Transitions}
The DP table for each node $t$ is computed based on the tables of its children.
For ease of notation, from here on, let $S_t = \{v \in X_t \mid s(v)=\mathtt{selected}\}$ and $U_t = X_t \setminus S_t$.

\subsubsection{Leaf Node}
For a tree decomposition, a leaf node $t$ has a bag $X_t$ containing a single vertex, say $X_t = \{v\}$.
We initialise the DP table for $t$ with two valid base profiles:
\begin{itemize}
    \item \textbf{Profile 1: $v$ is selected}. We set $\mathtt{s}(v)=\mathtt{selected}$, $\mathtt{FN}(v)=0$, and $\mathtt{O}(v)=0$.
    \item \textbf{Profile 2: $v$ is not selected}. We set $\mathtt{s}(v)=\mathtt{not}\text{-}\mathtt{selected}$ and $\mathtt{SAFE}(v)=0$.
\end{itemize}

\subsubsection{Introduce Node}
Let $t$ be an introduce node with child $t'$, such that $X_t = X_{t'} \cup \{z\}$.
For each valid profile $\lambda'$ on $X_{t'}$, we generate profiles $\lambda$ on $X_t$ by deciding the status of the new vertex $z$.
\begin{description}
    \item[Case 1: $s(z) = \mathtt{not}\text{-}\mathtt{selected}$.] The profile for vertices in $X_{t'}$ remains unchanged.
    We check if $z$ would \textit{seal} a $P_3$ (either fully within the bag, or if it is adjacent to a selected vertex whose
    $\mathtt{FN}$ bit is on).
    If so, we fix $\mathtt{SAFE}(z) = 1$, and otherwise fix $\mathtt{SAFE}(z) = 0$.

    \item[Case 2: $s(z) = \mathtt{selected}$.] \ \\
    \vspace{-0.4cm}\begin{enumerate}
        \item \textbf{$P_3$-freeness check}:
        The new profile is invalid if $z$ forms an induced $P_3$ with vertices already in the partial solution.
        This occurs whenever:
        \begin{itemize}
            \item $z$ is part of a $P_3$ contained entirely in the induced subgraph (either $z \sim u \sim v$ or $v \sim z \sim u$).
            \item $z$ is adjacent to a vertex $v \in S_{t'}$ that has a forgotten selected neighbour ($\mathtt{FN}(v)=1$).
            This would induce a $P_3$ where the non-edge $z \not\sim x$ is guaranteed by $x$ being forgotten.
        \end{itemize}
        \item \textbf{State update}: If all the checks pass, we create a new valid state $\lambda$.
        We initialise $\mathtt{FN}(z)=0$ and $\mathtt{O}(z)=0$.
        For any $v \in S_{t'}$ adjacent to $z$, its obligation $\mathtt{O}(v)$ is now fulfilled and is set to 0.
        For any $u \in U_{t'}$, we update its safety bit $\mathtt{SAFE}(u)$ to 1 if $z$ and some other vertices
        in $S_t$ form an induced $P_3$ with $u$.
    \end{enumerate}
\end{description}

\subsubsection{Forget Node}
Let $t$ be a forget node with child $t'$, such that $X_t = X_{t'} \setminus \{z\}$.
We project the profiles from $X_{t'}$ down to $X_t$.
\begin{description}
    \item[Case 1: $z \in S_{t'}$.] \ \\
    \vspace{-0.4cm}\begin{enumerate}
        \item \textbf{Obligation check}: The state is invalid if $\mathtt{O}(z)=1$.
        A vertex with an outstanding obligation cannot be forgotten.
        \item \textbf{State update}: If the check passes, we form the new profile on $X_t$.
        For each neighbour $v \in S_t$ of the forgotten vertex $z$, we update its forgotten neighbour bit: $\mathtt{FN}(v) \leftarrow \mathtt{FN}(v) \lor 1$.
    \end{enumerate}

    \item[Case 2: $z \in U_{t'}$.] \ \\
    \vspace{-0.4cm}\begin{enumerate}
        \item If $\mathtt{SAFE}(z)=0$, the omission of $z$ is as of yet unjustified. We must ensure that $z$ would seal a $P_3$ in the future.
        \begin{itemize}
            \item 
        The profile is invalid if the vertex $z$ has no selected neighbours in the bag ($N(z) \cap S_{t'} = \emptyset$).
        Without a selected neighbour, $z$ can never be part of an induced $P_3$. 
        \item The set of selected neighbours of $z$ in the bag forms a clique (otherwise it would seal a $P_3$ and the safety bit would be on). 
        \item \textbf{State update}: We place an obligation on all of $z$'s selected neighbours in the bag: for each $v \in N(z) \cap S_{t'}$, we set $\mathtt{O}(v) \leftarrow \mathtt{O}(v) \lor 1$.
    \end{itemize}
    \item If $\mathtt{SAFE}(z) = 1$ then no action is needed; the profile is projected to parent node, omitting $z$, and all the other vertices have the same status.
    \end{enumerate}
\end{description}

\subsubsection{Join Node}
Let $t$ be a join node with children $t_1$ and $t_2$, where $X_t = X_{t_1} = X_{t_2}$. A profile $\lambda$ for $X_t$ is valid if it can be formed by merging compatible valid profiles $\lambda_1$ from $t_1$ and $\lambda_2$ from $t_2$.

\begin{enumerate}
    \item \textbf{$P_3$-freeness check}: The merge is invalid if for any vertex $v \in S_t$, it has a forgotten neighbour from both branches ($\mathtt{FN}_1(v)=1$ and $\mathtt{FN}_2(v)=1$). This would create an induced $P_3$ with $v$ as the center and endpoints in $V_{t_1} \setminus X_t$ and $V_{t_2} \setminus X_t$.

    \item \textbf{State merge}: If the check passes, the new profile $\lambda$ is created by combining $\lambda_1$ and $\lambda_2$:
    \begin{itemize}
        \item For $u \in U_t$, the safety bit is merged with a logical OR, $\mathtt{SAFE}(u) = \mathtt{SAFE}_1(u) \lor \mathtt{SAFE}_2(u)$.
        \item For $v \in S_t$, the forgotten neighbour bit is also merged with a logical OR, $\mathtt{FN}(v) = \mathtt{FN}_1(v) \lor \mathtt{FN}_2(v)$.
        \item For $v \in S_t$, the obligation bit $O(v)$ is handled carefully. An obligation from one branch is discharged if the vertex acquires a forgotten neighbour from the other branch. The new obligation is set according to the rule:
        $\mathtt{O}(v) = (\mathtt{O}_1(v) \land \neg \mathtt{FN}_2(v)) \lor (\mathtt{O}_2(v) \land \neg \mathtt{FN}_1(v))$.
    \end{itemize}
\end{enumerate}

\subsection{Final Answer}

At the root node a profile is valid if for every outside vertex the safety bit is on, and for every selected vertex the obligation bit is off.
At this point we can extract the final answer.

\appendix
\bibliographystyle{alpha}
\bibliography{main}

\end{document}